\pdfoutput=1
\documentclass[11pt,a4paper]{article}

\usepackage[T1]{fontenc}
\usepackage[utf8]{inputenc}
\usepackage[a4paper,margin=2.7cm]{geometry}

\usepackage{amsmath}
\usepackage{amssymb}
\usepackage{amsthm}
\newtheorem{proposition}{Proposition}

\usepackage{graphicx}
\usepackage{float}
\usepackage{subcaption}
\usepackage{booktabs}
\usepackage{longtable}
\usepackage{lscape}
\usepackage{tabularx}
\usepackage{multirow}
\usepackage{xcolor}

\usepackage{algorithm}
\usepackage{algpseudocode}

\usepackage{authblk}

\usepackage[hidelinks]{hyperref}
\usepackage{url}

\usepackage{xparse}

\DeclareUnicodeCharacter{2212}{\ensuremath{-}}

\date{}

\ExplSyntaxOn

\cs_new_eq:NN \arxiv_original_author \author
\cs_new_eq:NN \arxiv_original_affil \affil
\cs_new_eq:NN \arxiv_original_abstract_begin \abstract
\cs_new_eq:NN \arxiv_original_abstract_end \endabstract
\cs_new_eq:NN \arxiv_original_bibliographystyle 

\prop_new:N \g_arxiv_affiliation_number_prop
\int_new:N  \g_arxiv_affiliation_number_int
\tl_new:N   \g_arxiv_abstract_tl
\tl_new:N   \g_arxiv_keywords_tl

\tl_new:N \l_arxiv_org_tl
\tl_new:N \l_arxiv_address_tl
\tl_new:N \l_arxiv_city_tl
\tl_new:N \l_arxiv_postcode_tl
\tl_new:N \l_arxiv_state_tl
\tl_new:N \l_arxiv_country_tl
\tl_new:N \l_arxiv_affiliation_text_tl

\keys_define:nn { arxiv / affiliation }
  {
    organization .tl_set:N = \l_arxiv_org_tl,
    addressline  .tl_set:N = \l_arxiv_address_tl,
    city         .tl_set:N = \l_arxiv_city_tl,
    postcode     .tl_set:N = \l_arxiv_postcode_tl,
    state        .tl_set:N = \l_arxiv_state_tl,
    country      .tl_set:N = \l_arxiv_country_tl
  }

\cs_new_protected:Npn \arxiv_get_affiliation_number:nN #1#2
  {
    \prop_get:NnNTF
      \g_arxiv_affiliation_number_prop
      {#1}
      #2
      { }
      {
        \int_gincr:N \g_arxiv_affiliation_number_int
        \tl_set:Nx #2
          { \int_use:N \g_arxiv_affiliation_number_int }

        \prop_gput:Nnx
          \g_arxiv_affiliation_number_prop
          {#1}
          { \int_use:N \g_arxiv_affiliation_number_int }
      }
  }

\cs_new_protected:Npn \arxiv_append_affiliation_field:V #1
  {
    \tl_if_blank:VF #1
      {
        \tl_if_blank:VF \l_arxiv_affiliation_text_tl
          {
            \tl_put_right:Nn
              \l_arxiv_affiliation_text_tl
              {,~}
          }

        \tl_put_right:NV
          \l_arxiv_affiliation_text_tl
          #1
      }
  }

\RenewDocumentCommand{\author}{O{}m}
  {
    \tl_clear:N \l_tmpa_tl

    \clist_map_inline:nn {#1}
      {
        \arxiv_get_affiliation_number:nN
          {##1}
          \l_tmpb_tl

        \tl_if_blank:VF \l_tmpa_tl
          {
            \tl_put_right:Nn \l_tmpa_tl {,}
          }

        \tl_put_right:NV
          \l_tmpa_tl
          \l_tmpb_tl
      }

    \tl_if_blank:nTF {#1}
      {
        \arxiv_original_author{#2}
      }
      {
        \use:x
          {
            \exp_not:N \arxiv_original_author
            [\exp_not:V \l_tmpa_tl]
            {\exp_not:n {#2}}
          }
      }
  }

\NewDocumentCommand{\affiliation}{O{}m}
  {
    \tl_clear:N \l_arxiv_org_tl
    \tl_clear:N \l_arxiv_address_tl
    \tl_clear:N \l_arxiv_city_tl
    \tl_clear:N \l_arxiv_postcode_tl
    \tl_clear:N \l_arxiv_state_tl
    \tl_clear:N \l_arxiv_country_tl
    \tl_clear:N \l_arxiv_affiliation_text_tl

    \keys_set:nn
      { arxiv / affiliation }
      {#2}

    \arxiv_append_affiliation_field:V \l_arxiv_org_tl
    \arxiv_append_affiliation_field:V \l_arxiv_address_tl
    \arxiv_append_affiliation_field:V \l_arxiv_city_tl
    \arxiv_append_affiliation_field:V \l_arxiv_postcode_tl
    \arxiv_append_affiliation_field:V \l_arxiv_state_tl
    \arxiv_append_affiliation_field:V \l_arxiv_country_tl

    \tl_if_blank:nTF {#1}
      {
        \use:x
          {
            \exp_not:N \arxiv_original_affil
            {\exp_not:V \l_arxiv_affiliation_text_tl}
          }
      }
      {
        \arxiv_get_affiliation_number:nN
          {#1}
          \l_tmpa_tl

        \use:x
          {
            \exp_not:N \arxiv_original_affil
            [\exp_not:V \l_tmpa_tl]
            {\exp_not:V \l_arxiv_affiliation_text_tl}
          }
      }
  }

\RenewDocumentEnvironment{abstract}{+b}
  {
    \tl_gset:Nn \g_arxiv_abstract_tl {#1}
  }
  { }

\NewDocumentEnvironment{highlights}{+b}
  { }
  { }

\NewDocumentEnvironment{keyword}{+b}
  {
    \tl_gset:Nn \g_arxiv_keywords_tl {#1}
  }
  { }

\NewDocumentCommand{\sep}{}
  {;\space}

\NewDocumentEnvironment{frontmatter}{}
  { }
  {
    \maketitle

    \arxiv_original_abstract_begin
      \tl_use:N \g_arxiv_abstract_tl
    \arxiv_original_abstract_end

    \par\medskip
    \noindent
    \textbf{Keywords: }
    \tl_use:N \g_arxiv_keywords_tl
    \par\bigskip
  }

\RenewDocumentCommand{\bibliographystyle}{m}
  {
    \arxiv_original_bibliographystyle{unsrt}
  }

\ExplSyntaxOff

\begin{document}

\begin{frontmatter}



\title{Beyond the Gegenbauer Paradigm: q-Orthogonal Kernels for Machine Learning}


\author[dept3]{Álvaro Sánchez-Paniagua Ríos}
\author[dept2]{Juan P. Llerena}
\author[dept3]{Alberto Lastra}
\author[dept4]{Nuria Torrado}
\author[dept3]{Edmundo J. Huertas}

\affiliation[dept3]{organization={Physics and Mathematics Department, University of Alcalá},
            country={Spain}}
            
\affiliation[dept2]{organization={Computer Science and Artificial Intelligence Department, University of Alcalá}, country={Spain}}

\affiliation[dept4]{organization={Mathematics Department, Universidad Autónoma de Madrid},
            country={Spain}}

            
           


\begin{abstract}
The performance of Support Vector Machines (SVMs) critically depends on the kernel function choice, which enables implicit mapping of data into high-dimensional feature spaces. While classical kernels like Radial Basis Function (RBF) remain popular, orthogonal polynomial kernels offer mathematically interpretable alternatives that can incorporate structured prior knowledge. This work extends the orthogonal polynomial kernel paradigm by introducing a novel family based on discrete $q$-Hermite I polynomials, a class of $q$-orthogonal polynomials that generalize classical Hermite polynomials through a deformation parameter $q$. We formally define the q-Hermite kernel and establish its validity under Mercer's theorem. The kernel's inherent boundedness properties naturally prevent annihilation and explosion effects without requiring explicit scaling mechanisms. Extensive experiments across 20 benchmark datasets demonstrate that the proposed kernel achieves competitive performance compared to both classical kernels and other orthogonal polynomial kernels, while offering advantages in numerical stability and computational simplicity. Our results confirm that $q$-orthogonal polynomials constitute a promising direction for kernel design, bridging mathematical elegance with practical machine learning applications, that provides conceptual and algorithmic resources that may be further extended to emerging quantum computing paradigms. To facilitate full reproducibility, we provide the complete implementation and experimental pipeline in an open-access GitHub repository at \url{https://github.com/Kokechacho/SVMs-QSVMs}.
\end{abstract}

\begin{highlights}
\item Novel SVM kernel based on $q$-Hermite $q$-orthogonal polynomials
\item The inherent boundedness of the q-polynomial prevents both annihilation and explosion effects without scaling functions
\item Theoretical validation under Mercer's theorem conditions
\item Competitive performance across 20 diverse benchmark datasets
\item Computational advantages over existing orthogonal polynomial kernels
\end{highlights}


\end{frontmatter}



\section{Introduction}
\label{secintro}

Support Vector Machines (SVMs) have emerged as one of the most powerful supervised learning algorithms for both binary and multiclass classification tasks \cite{chandra2021survey, huang2018applications}. The fundamental principle underlying SVMs is the kernel trick, which enables the efficient computation of inner products in high-dimensional (potentially infinite-dimensional) feature spaces without explicitly computing the feature mappings. According to Mercer's theorem \cite{mercer1909xvi}, a symmetric continuous function $K(x,y)$ defined on a compact set is a valid kernel function if and only if it is positive semi-definite, meaning that for any finite set of points and corresponding coefficients, the resulting Gram matrix is positive semi-definite. This theoretical foundation ensures that any kernel satisfying these conditions corresponds to an implicit feature space where SVMs can operate effectively.

The kernel trick transforms the SVM optimization problem into a dual formulation that depends only on kernel evaluations between training samples, rather than on explicit feature vectors \cite{shawe2011review}. This computational strategy has made SVMs applicable to a wide range of problems, including text classification, face detection, and biological sequence analysis \cite{goudjil2018novel, guo2000face, yang2004biological}. This versatility has even inspired extensions into emerging computational paradigms, most notably Quantum Support Vector Machines (QSVMs) \cite{akrom2024quantum, gentinetta2203complexity, hossain2025performance}, which leverage quantum circuits to compute kernel functions by exploiting the representational richness of quantum feature mappings and the natural ability of quantum systems to encode high-dimensional inner products \cite{park2020practical}.

The choice of kernel function plays a central role in determining the adaptability of SVMs to the underlying data structure, thereby directly influencing both their effectiveness and predictive accuracy \cite{khan2024ensemble}. Although linear kernels suffice for linearly separable datasets, real-world problems frequently require more expressive nonlinear transformations. Classical kernel families, most notably the Radial Basis Function (RBF) \cite{thurnhofer2020radial} and polynomial kernels, remain popular due to their flexibility and well-established theoretical foundations \cite{ben2008support}. Nevertheless, these kernels are not universally optimal across all data distributions, which motivates the exploration of more specialized and structurally informed kernel constructions \cite{du2024exploring, chapelle1999model, guido2024overview}.

Orthogonal polynomial kernels have emerged as a promising alternative, valued for their capacity to incorporate structured prior knowledge and for their inherent mathematical interpretability \cite{taheri2024orthogonal, moghaddam2024rational}. Initial developments in this direction relied on well-known families of orthogonal polynomials, including Legendre, Chebyshev, and Hermite polynomials, to construct valid Mercer kernels \cite{moghaddam2016new, padierna2018novel}. 

In this work, we extend this line of research by introducing $q$-orthogonal polynomials as a foundation for kernel design. The principal advantage of this approach lies in the enhanced flexibility and adaptability offered by the deformation parameter $q$, which enables fine-grained control over kernel properties without compromising mathematical tractability. Unlike classical orthogonal polynomial kernels with fixed recursive structures, $q$-orthogonal polynomials provide a continuous spectrum of kernel behaviors that can be tailored to specific data characteristics \cite{cao2023review, koekoek2010hypergeometric}. This mathematical framework bridges classical kernel methods with quantum-inspired computational paradigms, offering some advantages of quantum representations while remaining fully implementable on classical hardware \cite{riley2021bivariate}.

Specifically, we investigate the discrete $q$-Hermite I polynomial family, a class of $q$-orthogonal polynomials that generalizes classical Hermite polynomials through a deformation parameter $q$ \cite{asai2020deformed}. The distinctive advantage of these polynomials resides in their intrinsic boundedness and numerical stability properties, which directly address practical limitations of orthogonal polynomial kernels. Traditional kernels often suffer from numerical instability, including value explosion that requires ad hoc scaling mechanisms \cite{padierna2018novel}. In contrast, discrete $q$-Hermite I polynomials naturally maintain bounded values across their domain, enabling robust kernel construction that mitigates both annihilation and explosion phenomena without requiring external normalization. This property is especially beneficial for SVM applications that involve very large datasets or high-dimensional feature spaces, where computational efficiency becomes a limiting factor \cite{moghaddam2024rational}.

We, therefore, make four main contributions. First, we formally define the q-Hermite kernel and establish its validity under Mercer’s theorem. Second, we show that its inherent boundedness ensures numerical stability and removes the need for explicit scaling functions. Third, we develop a fully reproducible experimental pipeline, including hyperparameter optimization, cross-validation, and kernel evaluation routines, and we release it as an open-source implementation (see \hyperref[data]{Data Availability}). Finally, through extensive experiments, we demonstrate that the proposed kernel achieves competitive performance relative to both classical and other orthogonal polynomial kernels.

This paper is organized as follows: Section~\ref{secwork} reviews related work on orthogonal polynomial kernels, introducing the theory of $q$-orthogonal polynomials. Section~\ref{secalsal} details the construction of the q-Hermite kernel and its theoretical properties. Section~\ref{secmet} describes the experimental methodology, and Section~\ref{secexp} presents and discusses the results. Finally, Section~\ref{secconc} concludes the paper and outlines directions for future research.

\section{Related work}
\label{secwork}

The application of orthogonal polynomials to kernel construction represents an important research direction that extends previous work on kernel formulation and machine learning. Early work in this area explored individual families of classical orthogonal polynomials. Padierna et al. (2018) \cite{padierna2018novel} introduced a novel formulation of orthogonal polynomial kernels for SVM classifiers based on the Gegenbauer family of polynomials. This work unified and extended previous isolated studies of Legendre, Chebyshev, and Hermite polynomial kernels \cite{ye2006support, ozer2011set, zhao2013adaptive}, providing a systematic framework for constructing orthogonal polynomial-based kernels.

The Gegenbauer formulation addresses two critical issues that arise when using orthogonal polynomial kernels: the annihilation effect and the explosion effect. The annihilation effect occurs when kernel values become too small or zero, leading to a loss of discriminative information, while the explosion effect refers to the uncontrolled growth of kernel values outside specified domains. Padierna et al. \cite{padierna2018novel} resolved these issues by introducing a scaling function that constrains the polynomial amplitudes outside the natural domain $[-1,1]$, and a weight function specifically designed to maintain the stability of the kernel values \cite{tian2017some}. Their experimental results demonstrated that the Gegenbauer kernel family exhibits competitive accuracy with the widely-used RBF kernel while requiring significantly fewer support vectors, suggesting computational and interpretability advantages.

Formally, the Gegenbauer kernel is defined as

\begin{equation}
\label{eq:KGegen}
    K_{\mathrm{Geg}}(\textbf{x}, \textbf{z}) 
    = \prod_{j=1}^{d} 
    \sum_{i=0}^{n} 
    C^{\alpha}_{i}(x_j) \, C^{\alpha}_{i}(z_j) \, 
    w_{\alpha}(x_j, z_j) \, 
    u(C^{\alpha}_{i})^{2},
\end{equation}

where the normalization and weighting functions are given by

\begin{equation}
    u(C^{\alpha}_{i}) = \frac{1}{\sqrt{n+1}\,|C^{\alpha}_{i}(1)|},
\end{equation}

and

\begin{equation}
    w_{\alpha}(x, z) = 
    \big[(1-x^{2})(1-z^{2})\big]^{\frac{\alpha - 1}{2}} + \epsilon,
\end{equation}

with $\epsilon = 0.1$ when $\alpha>0.5$ ensuring numerical stability for inputs near the domain boundaries \cite{padierna2018novel}.

\subsection{q-Orthogonal Polynomials}
\label{secpoly}

The extension of orthogonal polynomials to q-analogues represents a profound mathematical development that has received limited attention in the kernel methods literature. Q-orthogonal polynomials (also called basic hypergeometric orthogonal polynomials) generalize classical polynomials by introducing a deformation parameter $q \in (0, 1)$, thereby creating a richer mathematical structure with remarkable properties. These q-polynomials satisfy q-difference equations instead of differential equations, yet inherit orthogonality properties and three-term recurrence relations from their classical counterparts. Importantly, certain families of q-orthogonal polynomials exhibit inherent boundedness properties that naturally prevent annihilation and explosion effects without requiring explicit scaling functions, potentially offering both mathematical elegance and computational advantages over classical polynomial kernels.

\subsubsection{Mathematical Foundations: q-Pochhammer and q-Binomials}
\label{subqp}

The theory of $q$-orthogonal polynomials represents a rich generalization of classical orthogonal polynomial theory. $q$-analogues emerge naturally in various mathematical contexts: the Rogers--Ramanujan identities in partition theory \cite{andrews1998theory}, quantum group theory \cite{kassel2012quantum}, and the representation theory of quantum algebras. The parameter $q$, typically chosen from the interval $(0, 1)$, acts as a deformation parameter: as $q \to 1$, $q$-orthogonal polynomials recover their classical (non-$q$) counterparts through limiting relationships \cite{koekoek2010hypergeometric}.

The $q$-Pochhammer symbol, defined as

\begin{equation}
    (a;q)_n = \prod_{m=0}^{n-1} (1-aq^m),
\end{equation}

serves as the fundamental building block for $q$-special functions. Its infinite product generalization is denoted by 

\begin{equation}
(a;q)_\infty = \prod_{m=0}^{\infty} (1-aq^m) \quad \text{for } |q|<1 
\end{equation}

The $q$-Pochhammer symbol satisfies crucial identities including the $q$-binomial theorem and the $q$-Vandermonde identity, which are $q$-analogues of classical binomial identities \cite{gasper2011basic}. These identities are essential for manipulating $q$-series and deriving properties of $q$-orthogonal polynomials.

The $q$-binomial coefficient is defined as

\begin{equation}
    \begin{bmatrix} n \\ k \end{bmatrix}_q = \frac{(q;q)_n}{(q;q)_k(q;q)_{n-k}}, \quad 0 \leq k \leq n,
\end{equation}

where $(q;q)_n$ represents the $q$-Pochhammer symbol with $a=q$ \cite{gasper2011basic, ernst2012comprehensive}. This coefficient satisfies remarkable properties: it reduces to the classical binomial coefficient $\binom{n}{k}$ as $q \to 1$, and it satisfies $q$-difference recurrence relations analogous to Pascal's triangle \cite{andrews1998theory}.

The classification of $q$-orthogonal polynomials is systematically organized through the extended Askey scheme (also called the $q$-Askey scheme), which provides a hierarchy connecting various families of $q$-orthogonal polynomials through limit relations \cite{koekoek2010hypergeometric}. This scheme includes families such as the $q$-Hahn polynomials (at the top), intermediate families, and limiting cases including $q$-Hermite and $q$-Laguerre polynomials. The structure of the Askey scheme enables understanding the relationships between different polynomial families and justifies the selection of promising candidates for kernel construction.

\section{The q-Hermite kernel functions}
\label{secalsal}

In this work, we introduce families of $q$-polynomials, specifically, the discrete $q$-Hermite I polynomials, as valid candidates for kernel formulation. In the following sections, their main properties are analyzed, and following the framework proposed by \cite{padierna2018novel}, a new family of kernels with its corresponding weight function is constructed.

\newcommand{\kerlinear}{K_{\text{Linear}}}
\newcommand{\kerpoly}{K_{\text{Poly}}}
\newcommand{\kerrbf}{K_{\text{RBF}}}
\newcommand{\kerhermite}{K_{\text{s-Herm}}}
\newcommand{\kergegen}{K_{\text{Gegen}}}
\newcommand{\keralsalam}{K_{\text{Q-HERMITE}}}

\begin{table}[htbp]
\centering
\caption{List of the kernels discussed in this work.}
\label{tab:orthogonal_kernels}
\small
\begin{tabularx}{\textwidth}{l l X}
\toprule
\textbf{Kernel} & \textbf{Short Label} & \textbf{Observations} \\
\midrule
\hypertarget{ker:linear}{Linear Kernel \cite{padierna2018novel}} 
& $K_{\text{Linear}}$
& Baseline kernel corresponding to the standard inner product; no feature mapping. \\

\hypertarget{ker:poly}{Polynomial Kernel \cite{padierna2018novel}} 
& $K_{\text{Poly}}$
& Extends the linear kernel by including polynomial interactions of degree $d$. \\

\hypertarget{ker:rbf}{RBF Kernel \cite{padierna2018novel}} 
& $K_{\text{RBF}}$
& Maps data into an infinite-dimensional feature space using a Gaussian basis function. \\

\hypertarget{ker:hermite}{Hermite Kernel \cite{padierna2018novel}} 
& $K_{\text{s-Herm}}$
& Slight modification from the baseline formulation, adjusting the scaling factor to $2^{-2i}$. \\

\hypertarget{ker:gegen}{Gegenbauer Kernel \cite{padierna2018novel}} 
& $K_{\text{Gegen}}$
& Derived from Gegenbauer polynomials; generalizes Legendre and Chebyshev kernels. \\

\hypertarget{ker:alsalam}{Q-Hermite Kernel \cite{hermoso2020second}} 
& $K_{\text{Q-HERMITE}}$
& Constructed from $q$-orthogonal discrete $q$-Hermite I polynomials. \\
\bottomrule
\end{tabularx}
\end{table}

\subsection{q-Hermite polynomials}

The discrete $q$-Hermite I polynomials \(H_n(x;q)\) form a family of \(q\)-orthogonal polynomials \cite{kim1997combinatorics}. They are characterized by the three-term recurrence
\begin{equation}\label{eq:recurrence}
x\,H_n(x;q) = H_{n+1}(x;q) + \beta_n\,H_n(x;q) + \gamma_n\,H_{n-1}(x;q), \qquad n \ge 0,
\end{equation}
with initial conditions \(H_0(x;q) = 1\), \(H_{-1}(x;q) = 0\), and coefficients
\[
\beta_n = (a + 1) q^n, \qquad \gamma_n = -a q^{n-1} (1 - q^n).
\]

The explicit form of the discrete $q$-Hermite I polynomials is 
\begin{equation}\label{eq:explicit}
H_n(x;q) = \sum_{i=0}^n
\begin{bmatrix} n \\ i \end{bmatrix}_q
(-1)^{n-i} q^{\frac{(n-i)(n-i-1)}{2}}
\prod_{j=0}^{i-1} (x - a q^j).
\end{equation}
Expanding the product and reordering the sums yields the equivalent simplified representation
\begin{equation}\label{eq:simplified}
H_n(x;q) = \sum_{i=0}^n
\begin{bmatrix} n \\ i \end{bmatrix}_q
(-1)^i x^{n-i}
\left(
\sum_{j=0}^i
\begin{bmatrix} i \\ j \end{bmatrix}_q
q^{\frac{(i-j)(i-j-1)}{2} + \frac{j(j-1)}{2}} a^j
\right).
\end{equation}

The discrete $q$-Hermite I polynomials \(H_n(x;q)\) depend on two parameters \(a\) and \(q\):
\[
a\in(-\infty,1),\qquad 0<q<1.
\]
The parameter $a$ determines the location of the roots of the discrete $q$-Hermite I polynomials and, consequently, the domain on which the associated kernel is naturally supported.  In this work we restrict ourselves to the case $a=-1$, which induces a well-structured and compact support on the interval $[-1,1]$, a domain that aligns naturally with standard data pre-processing practices in SVMs, where features are often scaled to a bounded range.  

This specialization is not merely a technical convenience: when $a=-1$, the discrete $q$-Hermite I polynomials form a distinguished subfamily that is closely connected to the $q$-Hermite polynomials, a class with deep relevance in quantum calculus and mathematical physics~\cite{gasper2011basic,koekoek2010hypergeometric}.

The parameter $q$ acts as the deformation parameter governing the transition between classical and $q$-deformed regimes.  Intermediate values of $q$ yield polynomial bases whose behavior interpolates smoothly between the classical Hermite case and strongly deformed limits, thus providing a flexible mechanism for adjusting the kernel’s approximation and regularization properties to the characteristics of each dataset.

Within this $a=-1$ setting, the limiting behavior is well understood: as $q \to 1^{-}$, the discrete $q$-Hermite I polynomials converge to the classical Hermite polynomials through established asymptotic relations~\cite{baker2000multivariable}. Thus, the kernel used in this study can be interpreted as a $q$-deformed Hermite kernel supported on $[-1,1]$, which is precisely the regime explored throughout this work.

For $a=-1$ then \(\beta_n=0\)  \(\forall n\), and
\[
\gamma_n = q^{\,n-1}(1-q^n),
\]
so the recurrence simplifies to
\begin{equation}\label{eq:recurrence-a=-1}
x\,H_n(x;q)=H_{n+1}(x;q)+q^{\,n-1}(1-q^n)\,H_{n-1}(x;q),\qquad n\ge0,
\end{equation}
with \(H_0(x;q)=1\), \(H_{-1}(x;q)=0\).

Expanding the recurrence from equation \eqref{eq:recurrence-a=-1} gives the first polynomials (up to \(n=4\)):
\[
\begin{aligned}
H_0(x;q) &= 1,\\[4pt]
H_1(x;q) &= x,\\[4pt]
H_2(x;q) &= x^2 - (1-q) = x^2 -1 + q,\\[6pt]
H_3(x;q) &= x^3 + (q^3-1)\,x,\\[6pt]
H_4(x;q) &= x^4 + (q^5+q^3-q^2-1)\,x^2 + (q^6-q^5-q^3+q^2).
\end{aligned}
\]

\begin{figure}[h!]
    \centering
    \includegraphics[width=\textwidth]{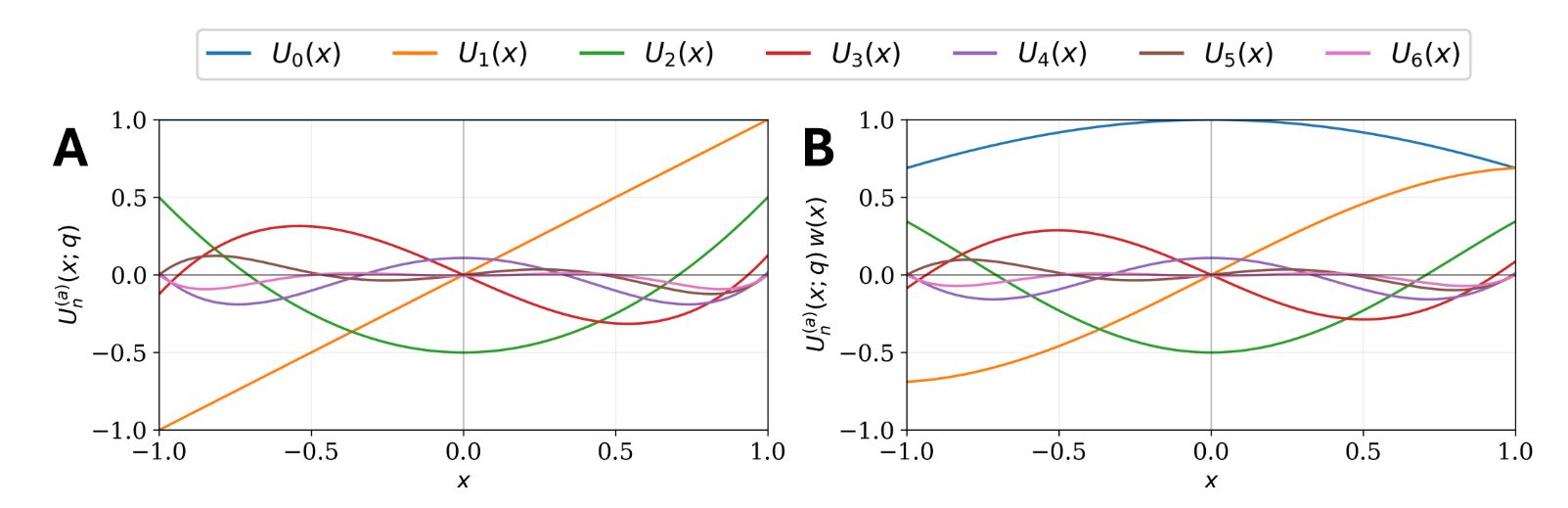}
    \caption{discrete $q$-Hermite I polynomials
\textbf{(A)} Polynomials without the weight function for \( q = 0.5 \). 
\textbf{(B)} Polynomials incorporating the weight function for \( q = 0.5 \).}
    \label{fig:conv}
\end{figure}

\subsubsection{The weight and scaling functions: kernel construction}

The weight function used in our case is taken from \cite{hermoso2020second}:

\begin{equation}
w_{a,q}(x) = (qx,a^{-1}qx;q)_\infty
\end{equation}

To obtain the bivariate version, we simply multiply the two univariate functions, as specified in \cite{xu2004discrete}, yielding:

\begin{equation}
w_{a,q}(x,z) = (qx,a^{-1}qx;q)\infty \cdot (qz,a^{-1}qz;q)\infty
\end{equation}

Substituting $a=-1$ then gives:

\begin{equation}
\label{eq:weightal}
w_{-1,q}(x,z) = (qx,-qx;q)\infty \cdot (qz,-qz;q)\infty
\end{equation}

\begin{figure}[h!]
    \centering
    \includegraphics[width=\textwidth]{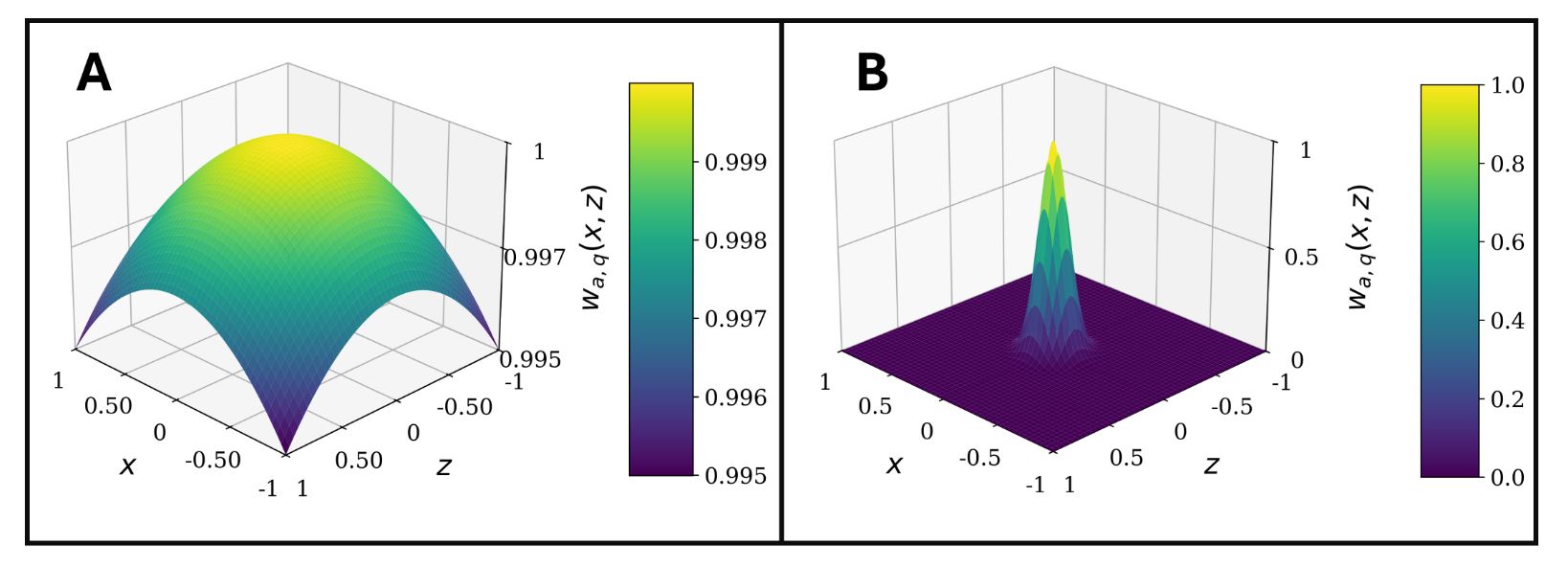}
    \caption{Bivariate weight function $w_{-1,q}(x,z)$ comparison for the q-Hermite polynomials. 
\textbf{(A)} Weight function for \( q = 0.05 \). 
\textbf{(B)} Weight function for \( q = 0.99 \).}
    \label{fig:weights}
\end{figure}

In the approach by \cite{padierna2018novel}, a scaling function is now required to constrain the amplitudes of Gegenbauer polynomials, which diverge beyond the interval $[-1,1]$ for certain parameter values. In contrast, as demonstrated in the following section, our method eliminates this necessity due to the inherent boundedness properties of the employed polynomial family. This yields significant advantages in both computational efficiency and mathematical simplicity.

Following the framework established in \cite{padierna2018novel}, our kernel is defined as:

\begin{equation}
\label{kalsalam}
    K_{\text{Q-HERMITE}}(\textbf{x},\textbf{z}) = \prod^d_{j=1}\sum^n_{i=0} H_{i}(x_j;q) H_{i}(z_j;q) \cdot w_{q}(x_j,z_j)
\end{equation}

\subsection{Theoretical properties}
\label{subthe}

According to Mercer's theorem \cite{mercer1909xvi}, a valid kernel function must be symmetric and positive semidefinite, ensuring that its expansion in terms of eigenfunctions and eigenvalues converges absolutely and uniformly.  To establish that our kernel construction satisfies these fundamental requirements, we examine the key properties that enter Mercer's hypothesis and explain briefly why each is relevant for both the theory and numerical practice.

First, it is essential that the polynomial basis employed in the kernel expansion remains uniformly bounded over the compact domain of interest. This condition can be satisfied either through the introduction of an explicit scaling function, as proposed in \cite{padierna2018novel}, or inherently, through the analytical properties of the kernel itself, as occurs in our case. Specifically, we consider the discrete $q$-Hermite I family (equivalently, the $a=-1$ specialization) with \(0<q<1\), for which we establish the following uniform boundedness property.


Before introducing our kernel construction, we record an explicit uniform
bound on the growth of the monic discrete $q$-Hermite~I polynomials on the
compact interval $[-1,1]$, which will be used later to control sup-norm
effects in the orthogonality regime.



\begin{proposition}
\label{prop:boundedness} Let $0<q<1$. The sequence of monic discrete $q$%
-Hermite~I polynomials $\{H_n(x;q)\}_{n=0}^{\infty}$ satisfies the following
bound. For every integer $n$, 
\begin{equation}  \label{eq:boundedness-qHermite}
\bigl\|H_n(\cdot;q)\bigr\|_{\infty,[-1,1]} \le \left[ \frac{5-3\sqrt{5}}{10}%
\left(\frac{1-\sqrt{5}}{2}\right)^n + \frac{5+3\sqrt{5}}{10}\left(\frac{1+%
\sqrt{5}}{2}\right)^n \right] \left(\frac{(q;q)_{\infty}}{(q;q)_{n}\,q^{n/2}}%
\right).
\end{equation}
\end{proposition}





\begin{proof}
For all $x\in\mathbb{R}$ and all positive integers $n$, the sequence $H_n(x;q)$ satisfies
the recurrence
\[
xH_n(x;q)=H_{n+1}(x;q)+\gamma_n\,H_{n-1}(x;q),
\]
with $\gamma_n=q^{n+1}(1-q^{n})\le 1$ for all $n$.
We define
\[
M_0=\sup_{x\in[-1,1]}|H_0(x;q)|=1,
\qquad
M_1=\sup_{x\in[-1,1]}|H_1(x;q)|=2.
\]
For all integers $n\ge 1$ and all $x\in[-1,1]$ one has
\[
|H_{n+1}(x;q)|
\le
|H_{n}(x;q)|+|\gamma_n|\,|H_{n-1}(x;q)|
\le
|H_{n}(x;q)|+|H_{n-1}(x;q)|.
\]
Therefore,
\[
\sup_{x\in[-1,1]}|H_{n+1}(x;q)|
\le
\sup_{x\in[-1,1]}|H_{n}(x;q)|
+
\sup_{x\in[-1,1]}|H_{n-1}(x;q)|.
\]
Let us define the recurrence $C(0)=1$, $C(1)=2$ and $C(n+1)=C(n)+C(n-1)$ for all $n\ge 1$.
It is clear from the definition of the sequence that
\[
C(n)\le \sup_{x\in[-1,1]}|H_{n}(x;q)|,\qquad n \ge 0,
\]
and the solution of $\{C(n)\}_{n\ge0}$ is given by
\[
C(n)=
\frac{5-3\sqrt{5}}{10}\left(\frac{1-\sqrt{5}}{2}\right)^n
+
\frac{5+3\sqrt{5}}{10}\left(\frac{1+\sqrt{5}}{2}\right)^n,
\qquad n\ge0.
\]
We take into account that
\[
\widetilde{H}_n(x;q)
=
\frac{H_n(x;q)}{\bigl\|H_n(x;q)\bigr\|_{L^2(\mu_q)}}
=
\frac{(q;q)_{\infty}}{(q;q)_n\,q^{\binom{n}{2}}}\,H_n(x;q),
\]
to conclude the proof.
\end{proof}



\subsubsection{Prevention of adverse annihilations and explosion effects}

Once Proposition~1 has been established, we tackle the weight function associated with the q-Hermite kernel. Improper handling of this weight may lead to the so-called \emph{explosion} or \emph{annihilation effects} \cite{ye2006support, ozer2011set}.  
In the work of \cite{padierna2018novel}, when introducing the Gegenbauer-based kernel, the authors modified its weight function by adding a small offset $\epsilon$, in order to bound it and prevent those adverse effects. 

In our formulation, such a correction is not required.  
This follows from the fact that we can analytically prove that our weight function remains strictly within the interval~$(0,1]$.  
Combined with Proposition~1, this property guarantees the absence of both annihilation and explosion effects.

Nevertheless, from a practical standpoint, to avoid the use of extremely small terms (near–zero), we follow the same strategy adopted by Padierna et al.~\cite{padierna2018novel}. Specifically, whenever the weight function satisfies $w_{-1,q}(x,z) < 0.1$, it is clipped to $w_{-1,q}(x,z) := 0.1$, which introduces negligible computational overhead.

\medskip
\noindent\textbf{Proposition 2.}  
For $a=-1$, $0<q<1$, and $x,z\in[-1,1]$, the weight function
\[
w_{-1,q}(x,z) = (q x, -q x; q)_\infty \cdot (q z, -q z; q)_\infty,
\]
satisfy
\[
 0 < w_{-1,q}(x,z) \le 1.
\]
\begin{proof}
We start from the elementary identity
\[
(1 - z q^k)(1 + z q^k) = 1 - z^2 q^{2k}, \quad \text{for all } k \ge 0.
\]
Taking the infinite product over $k$ gives the well-known relation between $q$-shifted factorials:
\[
(z;q)_\infty (-z;q)_\infty = (z^2;q^2)_\infty.
\]
Substituting $z = qx$ into this identity, we obtain
\[
w_{-1,q}(x) = (qx, -qx; q)_\infty = (q^2 x^2; q^2)_\infty = \prod_{m=1}^\infty (1 - q^{2m} x^2).
\]

Since $|x| \le 1$ and $0 < q < 1$, each factor in the product satisfies $0 < 1 - q^{2m} x^2 \le 1$.  
Consequently, every term of the infinite product lies strictly between $0$ and $1$, ensuring that $w_{-1,q}(x) \in (0,1]$.  
For the bivariate case, we recall that the scaling function is defined multiplicatively as
\[
w_{-1,q}(x,z) = w_{-1,q}(x) \, w_{-1,q}(z).
\]
Since both univariate factors satisfy $0 < w_{-1,q}(\cdot) \le 1$, their product does as well.  
Thus,
\[
0 < w_{-1,q}(x,z) \le 1, \quad \forall (x,z) \in [-1,1]^2
\]
\end{proof}


\begin{proposition}
\label{propWeight} For $0<q<1$ and $x,z\in[-1,1]$, the bivariate weight 
\begin{equation*}
w_{q}(x,z)=(qx,-qx;q)_{\infty}\,(qz,-qz;q)_{\infty}
\end{equation*}
satisfies 
\begin{equation*}
0<w_{q}(x,z)\le 1.
\end{equation*}
\end{proposition}

\begin{proof}
First, observe that the definition of $q$-shifted factorials guarantees that $w_q(x,z)$ is finite for $0\le x,z\le 1$. Indeed, the convergence of the infinite product is equivalent to the convergence of
$$\sum_{j=0}^{\infty}\ln((1-q^{j+1}x)(1+q^{j+1}x)(1-q^{j+1}z)(1+q^{j+1}z))\le \sum_{j=0}^{\infty}2\ln(1+\max\{x,z\}q^{j+1}),$$
the last series being convergent in virtue of the ratio test. 

It holds that for all positive integer $N$ that
$$0\le \prod_{j=0}^{N}(1-q^{j+1}x)(1+q^{j+1}x)(1-q^{j+1}z)(1+q^{j+1}z)=\prod_{j=0}^{N}(1-(q^{j+1}x)^2)(1-(q^{j+1}z)^2)\le 1,$$
due to $0\le x,z\le 1$, and $0<q<1$. Observe that
$$w_q(x,z)=(qx,-qx;q)_{\infty}\,(qz,-qz;q)_{\infty}=\lim_{N\to\infty}\prod_{j=0}^{N}(1-q^{j+1}x)(1+q^{j+1}x)(1-q^{j+1}z)(1+q^{j+1}z)$$
which entails that $0\le w_{q}(x,z)\le 1$. In addition to this, for all $0\le x,z\le 1$, the quantity $w_{q}(x,z)\neq0$ since none of the factors in its definition vanishes.
\end{proof}


\section{Discrete $q$-Hermite I kernel as a Mercer kernel}

Finally, the symmetry of the kernel is straightforward, while the remaining step consists in demonstrating that it is positive semidefinite.  
According to Mercer's theorem~\cite{mercer1909xvi}, let $K(x,z)$ be a symmetric and continuous function defined on the compact domain $[a,b]\times[a,b]$. For any function $g(\cdot)$ belonging to the space of continuous functions on $[a,b]$, a sufficient condition for $K$ to be positive semidefinite is:
\[
\int_a^b \int_a^b K(x, z)\, g(x)\, g(z)\, dx\, dz \ge 0.
\]

\begin{proof}

Same as Padierna et al. \cite{padierna2018novel}, as $K_{\text{Q-HERMITE}}(\textbf{x}, \textbf{z}  )= \prod_{j=1}^d K_{\text{Q-HERMITE}}(x_j, z_j)$ the scalar version of our kernel is expressed as: $K_{\text{Q-HERMITE}}(x,z) =\sum_{i=0}^n H_i(x;q)\, H_i(z;q)\, w_q(x,z)$

Substituting the latter scalar definition of the kernel,
\[
K_{\text{Q-HERMITE}}(x,z) = \sum^n_{i=0} H_{i}(x_j;q) H_{i}(z_j;q) \cdot w_{q}(x_j,z_j),
\]
and recalling that the weight function satisfies
\[
w_{q}(\mathbf{x},\mathbf{z}) = (q\mathbf{x}, -q\mathbf{x}; q)_\infty \cdot (q\mathbf{z}, -q\mathbf{z}; q)_\infty,
\]
we obtain:
\[
\int \int\sum_{i=0}^n
H_i(x;q)\, H_i(z;q)\,
(qx, -qx; q)_\infty \,
(qz, -qz; q)_\infty \,
g(x)\, g(z).
\]

Reordering the terms, we can separate the summations over \(x\) and \(z\):
\[
= \sum_{i=0}^n 
\left( \int H_i(x;q)\, (qx, -qx; q)_\infty\, g(x) \right)
\left( \int H_i(z;q)\, (qz, -qz; q)_\infty\, g(z) \right).
\]

We can express the previous expression more compactly as:
\[
\sum_{i=0}^n 
\left( \int H_i(x;q)\, (qx, -qx; q)_\infty\, g(x) \right)^2 \ge 0.
\]

Therefore, the kernel \(K_{\text{Q-HERMITE}}\) is positive semidefinite and hence constitutes a valid Mercer kernel.

Finally, since the product of positive semidefinite kernels is also positive semidefinite, this construction can be naturally extended to higher dimensions as:
\[
K_{\text{Q-HERMITE}}(\mathbf{x},\mathbf{z}) = \prod_{j=1}^d K_{\text{Q-HERMITE}}(x_j, z_j).
\]
\end{proof}

\subsection{Computational Algorithm}

Algorithm \ref{alg:alsalam} ilustrates how we build this kernel. We compute the q-Hermite kernel by leveraging the three-term recurrence relation of the discrete $q$-Hermite I polynomials (from \ref{eq:recurrence-a=-1}) to generate orthogonal basis functions up to degree $N$ for each input dimension. This approach follows the established methodology for constructing kernels from orthogonal polynomial systems \cite{padierna2018novel, moghaddam2016new, zhou2007constructing}, where the explicit feature map is formed by evaluating the polynomial basis and applying the corresponding orthogonality weight function. The recurrence enables stable computation of polynomial values while maintaining the orthogonality properties essential for kernel validity.

The multivariate kernel is constructed as a product of univariate kernels across dimensions, consistent with separable kernel design principles common in kernel methods. For each dimension, we compute weighted feature maps by combining the polynomial evaluations with the q-Hermite weight function, then form the dimension-wise kernel via inner products. The final kernel matrix results from the product of these dimensional kernels, ensuring positive definiteness through the multiplicative composition of valid univariate kernels. This explicit feature map construction provides computational advantages for moderate-dimensional data while guaranteeing Mercer conditions through the orthogonal polynomial framework.

\begin{algorithm}[H]
\caption{Computation of the q-Hermite kernel function $K_{\text{Q-HERMITE}}(X, Z, q, a, N)$}
\label{alg:alsalam}

\textbf{Inputs:} $X, Z \in [-1,1]^d$, $q \in (0,1)$, $a =-1$, $N \in \mathbb{Z}^+$ \\
\textbf{Output:} Kernel matrix $K_{\text{Q-HERMITE}}(X, Z) \in \mathbb{R}^{n_X \times n_Z}$

\begin{algorithmic}[1]
\State \textbf{Initialize:} $(n_X, d) \gets \text{shape}(X)$, $(n_Z, \_) \gets \text{shape}(Z)$
\State $K \gets \mathbf{1}_{n_X \times n_Z}$ \Comment{Initialize kernel as an all-ones matrix}
\For{$dim = 1$ \textbf{to} $d$}
    \State Compute $H_x \gets H(X_{:,dim}; q, N)$ and $H_z \gets H(Z_{:,dim}; q, N)$ \Comment{From (\ref{eq:recurrence-a=-1})}
    \State Compute weights $w_x \gets w_{q}(X_{:,dim})$ and $w_z \gets w_{q}(Z_{:,dim})$ \Comment{From (\ref{eq:weightal})}
    \State Build feature matrices $\phi_x \gets H_x \odot w_x$ and $\phi_z \gets H_z \odot w_z$
    \State Compute dimension-wise kernel $K_{\text{dim}} \gets \phi_x \phi_z^\top$
    \State Update global kernel $K \gets K \odot K_{\text{dim}}$
\EndFor
\State \textbf{return} $K$
\end{algorithmic}
\end{algorithm}

\section{Experimental Methodology}
\label{secmet}

In this section, we evaluate the performance of our newly proposed q-hermite kernel (\ref{kalsalam}) against both classical kernels (\hyperlink{ker:linear}{$K_{\text{Linear}}$}, \hyperlink{ker:poly}{$K_{\text{Poly}}$} and \hyperlink{ker:rbf}{$K_{\text{RBF}}$})  and other orthogonal polynomial-based kernels (\hyperlink{ker:hermite}{$K_{\text{s-Herm}}$} and \hyperlink{ker:gegen}{$K_{\text{Gegen}}$}), both formally defined in \cite{padierna2018novel}. The comparison is carried out between the datasets shown in Table \ref{tab:dataset_summary} and evaluation metrics (Subsection \ref{submetrics}). Note that only 7 out of those 20 datasets overlap with Padierna's work \cite{padierna2018novel}. We use these 7 merely as an arbitrary baseline to obtain comparable reference values, while incorporating an additional 13 datasets to increase heterogeneity and sample diversity, including 7 multiclass datasets among them.

\begin{table}[htbp]
\centering
\small
\caption{Summary of the datasets used in this study}
\label{tab:dataset_summary}
\begin{tabular}{lrrrr}
\toprule
Dataset & Samples & Features & Classes & Imbalance \\
\midrule
ecoli$^{1}$        & 336  & 7  & 8  & 0.014 \\
glass$^{1}$        & 214  & 9  & 6  & 0.118 \\
SPECTF$^{1}$       & 267  & 44 & 2  & 0.259 \\
student$^{1}$      & 145  & 31 & 8  & 0.229 \\
parkinson$^{1}$    & 195  & 22 & 2  & 0.327 \\
vertebral$^{1}$    & 310  & 6  & 3  & 0.400 \\
fertility$^{1}$    & 100  & 9  & 2  & 0.136 \\
cervical$^{1}$     & 72   & 19 & 2  & 0.412 \\
mines$^{1}$        & 338  & 3  & 5  & 0.915 \\
blood$^{1}$        & 748  & 4  & 2  & 0.312 \\
statlog$^{1}$      & 845  & 18 & 4  & 0.917 \\
hepatitis$^{1}$    & 80   & 19 & 2  & 0.194 \\
wine$^{1}$         & 178  & 13 & 3  & 0.676 \\
haberman$^{1}$     & 306  & 3  & 2  & 0.360 \\
ionosphere$^{1}$   & 351  & 34 & 2  & 0.560 \\
breast$^{1}$       & 683  & 9  & 2  & 0.538 \\
australian$^{2}$   & 690  & 15 & 2  & 0.802 \\
german$^{2}$       & 1000 & 25 & 2  & 0.429 \\
heart$^{2}$        & 270  & 14 & 2  & 0.800 \\
fourclass$^{2}$    & 862  & 3  & 2  & 0.553 \\
\bottomrule
\end{tabular}

\vspace{4pt}
\caption*{\footnotesize Imbalance ratio is computed as $\text{minority class}/\text{majority class}$, both for binary and multiclass datasets.}
\vspace{3pt}
\caption*{\footnotesize $^{1}$\,UCI Machine Learning Repository: \texttt{https://archive.ics.uci.edu/ml/datasets.html}}  
\caption*{\footnotesize $^{2}$\,LIBSVM Data Repository: \texttt{https://www.csie.ntu.edu.tw/\textasciitilde cjlin/libsvmtools/datasets/binary.html}}
\end{table}

The first step in this methodology is the partitioning of each dataset into a 70-30 proportion, ensuring we have sufficient data for training and enough for validating the results, aligning with established practices in the literature. Each kernel's hyperparameters were optimized through a Bayesian optimization strategy using the Tree-structured Parzen Estimator (TPE) algorithm \cite{falkner2018bohb} (as implemented in the Optuna framework via its default TPESampler \cite{ozaki2025optunahub}), designed to iteratively explore the search space and converge toward configurations yielding the highest classification performance.


For each kernel–dataset pair, a total of 100 optimization trials were conducted, each corresponding to a distinct set of hyperparameters. The quality of a configuration was quantified by the mean classification accuracy obtained by training an SVM and evaluating it using a 10-fold, stratified cross-validation scheme that ensured balanced representation of all classes across folds. All datasets were scaled to the interval $[-1, 1]$ to satisfy the domain requirements of orthogonal polynomial kernels and to prevent bias toward features with larger magnitudes. Scaling was performed within the cross-validation loop, with the Min–Max parameters fitted on the training folds and applied to the corresponding validation folds to avoid data leakage.

To guarantee statistical robustness and mitigate the influence of random data partitioning, the complete partitioning, optimization and evaluation procedure was repeated 35 times per dataset. A different random seed was used for each repetition, generating an independent partition and 10-fold cross-validation split. Consistent conditions were kept for all kernels and hyperparameter configurations within the same repetition by sharing the same fold assignments.


The final performance of each kernel was then reported as the average across these 35 independent repetitions. This repeated approach with Bayesian optimization follows recommended methodologies for avoiding overfitting in hyperparameter \cite{bergstra2011algorithms, feurer2019hyperparameter}, enhancing the methodology proposed by Padierna et al. \cite{padierna2018novel}.

The hyperparameter search spaces were carefully designed based on established practices in SVM literature and preliminary experimentation. The regularization parameter $C$ was bounded between 0.001 and 100 \cite{hastie2004entire}. For polynomial-based kernels including the proposed q-hermite kernel, the degree parameter $n$ was limited to integers from 1 to 6, as higher degrees often induce overfitting without substantial performance gains \cite{ali2007optimal}. The RBF kernel parameter $\gamma$ was explored in the range $[2^{-6}, 2^{2}]$, matching the search space of the polynomial kernel scale parameter $\gamma$, while the offset term (coef0) was fixed to 1. The $q$-deformation parameter for the q-Hermite kernel was constrained to $(0,1)$ (specifically $[0.01, 0.99]$ to avoid numerical instability at the boundaries). Similarly, the Gegenbauer parameter $\alpha$ was bounded to $(-0.5,1.5]$ (specifically $[-0.49, 1.5]$ to avoid the boundary $-0.5$).

All experiments were conducted on a high-performance server hosting two AMD ROME 7552 processors (48 cores and 96 threads each). The system integrates 640 GB of DDR4 ECC RAM and a 2 TB SSD used for high-speed data access. For accelerated computation, the platform includes six NVIDIA Quadro RTX 5000 GPUs (16 GB each). The infrastructure also incorporates a 9 kVA UPS unit to ensure electrical stability and enable safe system shutdown during power outages. The implementation utilized Python 3.9 with the following specific library versions: NumPy 1.23.0 for numerical operations, pandas 1.5.0 for data manipulation, SciPy 1.9.0 for statistical computations and special functions, scikit-learn 1.3.0 for SVM implementation and evaluation metrics, Optuna 3.0.0 for Bayesian hyperparameter optimization, and ucimlrepo 0.0.7 for dataset loading from the UCI Machine Learning Repository (Table \ref{tab:dataset_summary}). To support the reader in setting up test environments, if they wish to replicate results, the software provide includes the set of dependencies necessary to automate this process. For each experimental trial, computational time and all necessary metrics are saved to CSV files. The complete experimental methodology is summarized in Algorithm~\ref{alg:pipeline}.

\begin{algorithm}[htbp]
\caption{Experimental pipeline}
\label{alg:pipeline}
\textbf{Inputs:} dataset collection $\mathcal{D}$, kernel set $\mathcal{K}$, 
$T$ experimental trials, $N$ optimization trials.\\
\textbf{Output:} aggregated performance summary $\mathcal{S}$
\begin{algorithmic}[1]
\State \textbf{Initialize:} $\mathcal{S}\gets\varnothing$
\State $(\forall D\in\mathcal{D})\ (X,y)\gets\textsc{Load\_Data}(D)$
\ForAll{dataset $D\in\mathcal{D}$}
  \ForAll{kernel $k\in\mathcal{K}$}
    \State $\mathcal{M}\gets\varnothing$
    \For{$t\gets 1$ \textbf{to} $T$} \Comment{$T = 35$}
      \State $(X_s,X_e,y_s,y_e)\gets\textsc{StratifiedSplit}(X,y,0.7,\text{seed}=t)$
      \State $\theta^\star\gets\textsc{BayesianOptimize}(k,X_s,y_s,N,\text{seed}=t)$ \Comment{$N = 100$}
      \State $m^\star\gets\textsc{TrainEvalSVM}(k,X_s,y_s,X_e,y_e,\theta^\star)$
      \State $\mathcal{M}\gets\mathcal{M} \cup \{m^\star\}$
    \EndFor
    \State $\bar{m}\gets\textsc{Aggregate}(\mathcal{M})$
    \State $\mathcal{S}\gets\mathcal{S}\cup\{(D,k,\bar{m})\}$
  \EndFor
\EndFor
\State \textbf{return} $\mathcal{S}$
\end{algorithmic}
\end{algorithm}

\begin{algorithm}[htbp]
\caption{Hyperparameter optimization with fixed folds (BayesianOptimize function)}
\label{alg:bayesian-minimal}
\textbf{Input:} kernel $k$, data $(X_s,y_s)$, trials $N$, seed $t$ \\
\textbf{Output:} optimal hyperparameters $\theta^\star$
\begin{algorithmic}[1]
\State $\mathcal{F} \gets \textsc{MakeFolds}(X_s, y_s, F=10, \text{seed}=t)$
\For{$i \gets 1$ \textbf{to} $N$}
    \State $\theta_i \gets \textsc{SampleHyperparams}()$
    \State $a_i \gets \textsc{CrossValidate}(k, \theta_i, \mathcal{F})$
\EndFor
\State $\theta^\star \gets \arg\max_i a_i$
\State \textbf{return} $\theta^\star$
\end{algorithmic}
\end{algorithm}

\subsection{Evaluation Metrics}
\label{submetrics}

We evaluated model performance using multiple metrics to comprehensively assess classification effectiveness and computational efficiency. Accuracy was employed as the primary performance measure, representing the proportion of correctly classified instances, accompanied by its standard deviation to quantify variability across experimental trials. The F1-score provided additional insight by offering a balanced perspective on classification performance, particularly valuable for imbalanced datasets. Model complexity was assessed through the proportion of support vectors (PSV) relative to set size. Computational requirements were quantified by measuring training time. For all metrics, we reported comprehensive statistics including mean, standard deviation, median, and 95\% confidence intervals based on multiple experimental trials to ensure robust statistical characterization of each kernel's performance characteristics. These metrics are formally defined in Table \ref{tab:metrics}.

\begin{table}[htbp]
\centering
\scriptsize
\setlength{\tabcolsep}{3pt}
\caption{Formal definition of evaluation metrics and aggregation methods}
\label{tab:metrics}
\begin{tabular}{p{0.25\textwidth}p{0.35\textwidth}p{0.33\textwidth}}
\toprule
\textbf{Metric} & \textbf{Per-Fold Computation} & \textbf{Cross-Trial Statistics} \\
\midrule

Accuracy\textsuperscript{1} &
$\displaystyle
\text{Acc}^{(k)} = \frac{1}{N_{\text{test}}^{(k)}} \sum_{i=1}^{N_{\text{test}}^{(k)}} 1(y_i^{(k)} = \hat{y}_i^{(k)})
$ &
$\displaystyle \mu_{\text{Acc}} = \frac{1}{T}\sum_{t=1}^{T} \frac{1}{K}\sum_{k=1}^{K} \text{Acc}^{(k,t)}$
\\

F1-Score \textsuperscript{2} &
$\displaystyle
\mathrm{F1}^{(k)} 
=
\frac{2\,\mathrm{TP}^{(k)}}{2\,\mathrm{TP}^{(k)} + \mathrm{FP}^{(k)} + \mathrm{FN}^{(k)}}
$
&
$\displaystyle 
\mu_{\mathrm{F1}}
=
\frac{1}{T}
\sum_{t=1}^{T}
\frac{1}{K}
\sum_{k=1}^{K}
\mathrm{F1}^{(k,t)}
$
\\

PSV\textsuperscript{3} &
$\displaystyle
\text{PSV}^{(k)} = \frac{\sum_{c=1}^{C} n_{\text{SV},c}^{(k)}}{N_{\text{total}}}
$ &
$\displaystyle \mu_{\text{PSV}} = \frac{1}{T}\sum_{t=1}^{T} \frac{1}{K}\sum_{k=1}^{K} \text{PSV}^{(k,t)}$
\\

Training Time\textsuperscript{4} &
$\displaystyle
t_{\text{train}}^{(k)} = t_{\text{SVM}}^{(k)}
$ &
$\displaystyle \mu_t = \frac{1}{T}\sum_{t=1}^{T} \frac{1}{K}\sum_{k=1}^{K} t_{\text{train}}^{(k,t)}$
\\
\bottomrule
\end{tabular}

\vspace{4pt}
\raggedright
\tiny
\textsuperscript{1}Accuracy: \url{https://scikit-learn.org/stable/modules/generated/sklearn.metrics.accuracy_score.html}\\
\textsuperscript{2}F1-Score: \url{https://scikit-learn.org/stable/modules/generated/sklearn.metrics.f1_score.html}\\
\textsuperscript{3}Support vectors: \url{https://scikit-learn.org/stable/modules/generated/sklearn.svm.SVC.html#sklearn.svm.SVC.n_support_}\\
\textsuperscript{4}Training time: \url{https://scikit-learn.org/stable/modules/generated/sklearn.svm.SVC.html#sklearn.svm.SVC.fit}

\vspace{4pt}
\raggedright
\tiny
\end{table}

where $K = 10$: cross-validation folds per trial, $T = 35$: experimental trials, $N_{\text{test}}^{(k)}$: test set size in fold $k$, $y_i^{(k)}$, $\hat{y}_i^{(k)}$: true and predicted labels, $n_{\text{SV},c}^{(k)}$: support vectors for class $c$ in fold $k$, $C$: number of classes, $N_{\text{total}}$: total dataset size, $\mathrm{TP}^{(k)}$: true positives in fold $k$, $\mathrm{FP}^{(k)}$: false positives in fold $k$, $\mathrm{FN}^{(k)}$: false negatives in fold $k$. Final reported statistics include mean ($\mu$), standard deviation ($\sigma$), median, and 95\% confidence intervals across the $T = 35$ trials.

\section{Experimental results and discussion}
\label{secexp}

The tables \ref{tab:accuracy_results}, \ref{tab:f1_results}, 
\ref{tab:training_time_results} and \ref{tab:support_vector_results}
summarize the main outcomes of the methodology described above, covering the four primary evaluation metrics.

Regarding accuracy (Table \ref{tab:accuracy_results}), we observe that the Gegenbauer kernels consistently remain at the top, whereas our q-hermite kernel typically occupies a top-middle position in the ranking (ranking third overall). The F1-score follows the same  trend (Table \ref{tab:f1_results}), showing a similar relative ordering across kernels. We observe that the proposed kernel $K_{qH}$ outperforms the Gegenbauer kernel in terms of accuracy on 9 out of the 20 datasets, i.e., in nearly half of the cases. As shown in Figure~\ref{fig:comparison}, this corresponds to an average accuracy discrepancy of -1.373\%. Furthermore, no statistically significant correlation is found between this performance difference and either the number of samples or the number of features in the datasets, as the corresponding p-values are well above conventional significance thresholds.

\begin{figure}[htbp]
    \centering
    
    \begin{subfigure}[b]{\textwidth}
        \centering
        \includegraphics[width=0.95\textwidth]{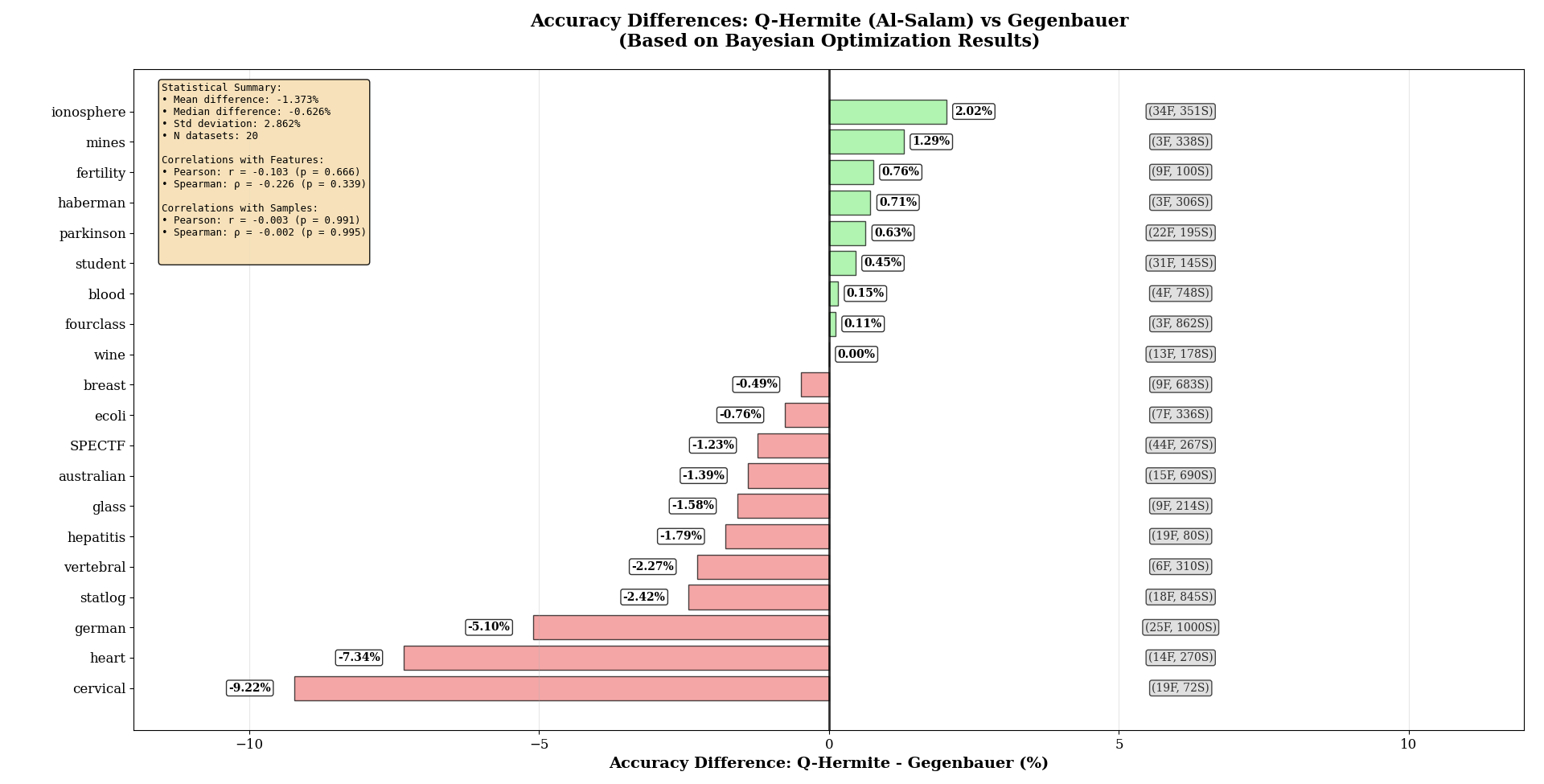}
        \caption{Accuracy comparison between Q-Hermite and Gegenbauer kernels across different datasets. Green bars indicate Q-Hermite advantage, red bars indicate Gegenbauer advantage.}
        \label{fig:bars}
    \end{subfigure}
    
    \vspace{0.5cm} 
    
    \begin{subfigure}[b]{\textwidth}
        \centering
        \includegraphics[width=0.75\textwidth]{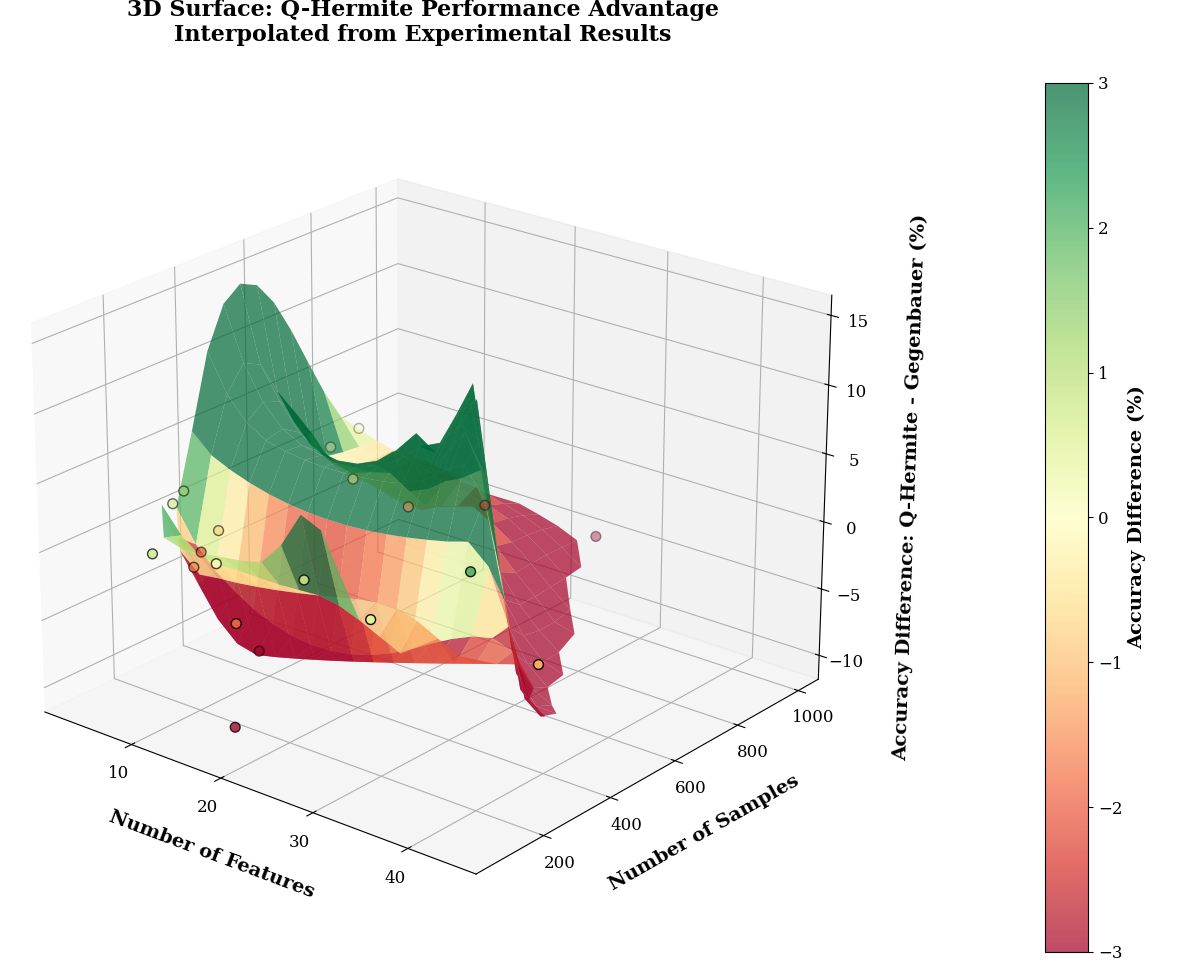}
        \caption{Interpolated 3D surface showing Q-Hermite advantage as a function of the number of features and samples.}
        \label{fig:3d}
    \end{subfigure}
    
    \caption{Comparative analysis of Q-Hermite and Gegenbauer kernel performance across multiple datasets. (a) Percentage accuracy differences between the two kernels. (b) 3D representation of Q-Hermite advantage interpolated from experimental results.}
    \label{fig:comparison}
\end{figure}

For training time (Table \ref{tab:training_time_results}), we obtain a substantial advantage over the Gegenbauer family, achieving better performance in 16/20 datasets. This improvement arises not only from the convenient absence of a scaling function, but also from the ability to modulate the resolution of our weight function by adjusting the truncation bound of the Pochhammer operator. For instance, the weight function $w_{-1,q}(x,z)=(qx,-qx;q)_{\infty}\allowbreak(qz,-qz;q)_{\infty}$ can be replaced with its truncated version $w_{-1,q}(x,z)=(qx,-qx;q)_{\vartheta}\allowbreak(qz,-qz;q)_{\vartheta}$. Although the training time does not match that of the simplest classical kernels, it represents a clear improvement over the orthogonal polynomial paradigm, specifically over the Gegenbauer kernel.

\begin{figure}[htbp]
    \centering
    \includegraphics[width=1\textwidth]{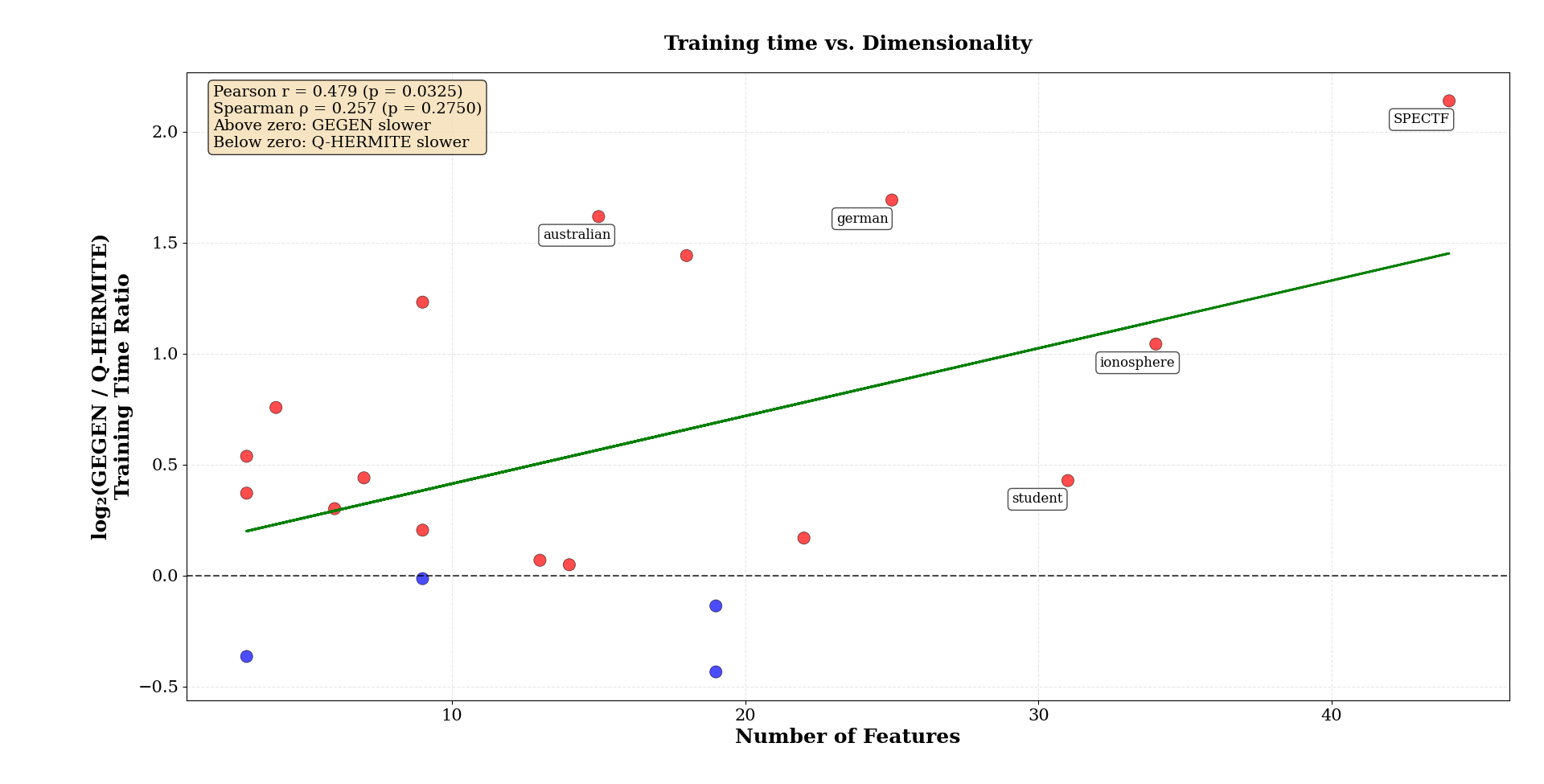}
    \caption{Training-time ratio $ \log_{2}(\text{GEGEN}/\text{Q-HERMITE}) $ across 20 UCI datasets as a function of feature dimensionality. Positive values (red dots) indicate a speed advantage for $K_{qH}$ for that dataset (while blue dots represent datasets where $K_{Gegen}$ is faster). A linear green fit reveals a significant correlation ($r=0.479$, $p=0.0325$) between computational gains with dimensionality.}
    \label{fig:NF}
\end{figure}

Across the benchmark, the q-Hermite kernel displays a clear separation between its computational and accuracy behavior. Figure \ref{fig:NF} shows that its training-time advantage grows systematically with dataset dimensionality, exhibiting a significant correlation ($r=0.479$, $p=0.0325$), which indicates that the kernel becomes increasingly efficient in high-dimensional regimes. In contrast, Figure \ref{fig:bars} reveals that accuracy differences relative to the Gegenbauer kernel remain small (mean: $-1.373\%$, median: $-0.626\%$) and do not correlate with dimensionality (nor number of samples), suggesting no systematic accuracy degradation as feature count increases. Overall, the results indicate that $K_{qH}$ offers substantial computational savings, particularly in high-dimensional settings, while maintaining reasonable accuracy levels.

Finally, the support vector proportion (Table \ref{tab:support_vector_results}) yields an advantage for classical algorithms, partially contradicting the results reported by Padierna et al \cite{padierna2018novel}. In our case, this metric reveals an interesting dissociation between model complexity and predictive performance for our Q-Hermite kernel (and even with the Gegenbauer kernel). Despite exhibiting among the highest support vector proportions in some cases, its optimal predictive performance does not require support vector minimization. For instance, on datasets like Ionosphere and Fertility, Q-Hermite produces among the poorest support vector counts while delivering superior accuracy and F1-score compared to all other kernels.

\begin{table}[htbp]
\centering
\caption{Statistics on kernel performance according to classification accuracy}
\label{tab:accuracy_results}
\resizebox{\textwidth}{!}{%
\begin{tabular}{lcccccc}
\toprule
\textbf{Kernel} & \textbf{Q-HERMITE} & \textbf{LINEAR} & \textbf{POLY} & \textbf{RBF} & \textbf{HERMITE} & \textbf{GEGEN} \\
\midrule
\textbf{Dataset} & Mean $\mid$ Std & Mean $\mid$ Std & Mean $\mid$ Std & Mean $\mid$ Std & Mean $\mid$ Std & Mean $\mid$ Std \\
\midrule
SPECTF & \begin{tabular}{@{}c@{}}77.57 $\pm$ 2.44\\5\end{tabular} & \begin{tabular}{@{}c@{}}77.50 $\pm$ 3.71\\6\end{tabular} & \begin{tabular}{@{}c@{}}78.06 $\pm$ 2.82\\3\end{tabular} & \begin{tabular}{@{}c@{}}77.74 $\pm$ 2.51\\4\end{tabular} & \begin{tabular}{@{}c@{}}\textbf{79.01 $\pm$ 0.00}\\1\end{tabular} & \begin{tabular}{@{}c@{}}78.80 $\pm$ 3.70\\2\end{tabular} \\
\hline
australian & \begin{tabular}{@{}c@{}}84.46 $\pm$ 1.82\\6\end{tabular} & \begin{tabular}{@{}c@{}}85.53 $\pm$ 1.94\\4\end{tabular} & \begin{tabular}{@{}c@{}}85.69 $\pm$ 1.75\\3\end{tabular} & \begin{tabular}{@{}c@{}}85.73 $\pm$ 2.08\\2\end{tabular} & \begin{tabular}{@{}c@{}}85.34 $\pm$ 1.40\\5\end{tabular} & \begin{tabular}{@{}c@{}}\textbf{85.85 $\pm$ 1.86}\\1\end{tabular} \\
\hline
blood & \begin{tabular}{@{}c@{}}78.29 $\pm$ 1.29\\2\end{tabular} & \begin{tabular}{@{}c@{}}76.15 $\pm$ 0.34\\6\end{tabular} & \begin{tabular}{@{}c@{}}78.04 $\pm$ 1.47\\4\end{tabular} & \begin{tabular}{@{}c@{}}\textbf{78.32 $\pm$ 1.60}\\1\end{tabular} & \begin{tabular}{@{}c@{}}76.93 $\pm$ 0.80\\5\end{tabular} & \begin{tabular}{@{}c@{}}78.13 $\pm$ 1.55\\3\end{tabular} \\
\hline
breast & \begin{tabular}{@{}c@{}}96.63 $\pm$ 1.10\\5\end{tabular} & \begin{tabular}{@{}c@{}}96.81 $\pm$ 0.76\\4\end{tabular} & \begin{tabular}{@{}c@{}}96.88 $\pm$ 1.13\\3\end{tabular} & \begin{tabular}{@{}c@{}}96.46 $\pm$ 0.91\\6\end{tabular} & \begin{tabular}{@{}c@{}}97.02 $\pm$ 0.95\\2\end{tabular} & \begin{tabular}{@{}c@{}}\textbf{97.11 $\pm$ 0.88}\\1\end{tabular} \\
\hline
cervical & \begin{tabular}{@{}c@{}}80.00 $\pm$ 6.17\\5\end{tabular} & \begin{tabular}{@{}c@{}}\textbf{91.95 $\pm$ 5.65}\\1\end{tabular} & \begin{tabular}{@{}c@{}}90.78 $\pm$ 5.38\\2\end{tabular} & \begin{tabular}{@{}c@{}}90.00 $\pm$ 5.07\\3\end{tabular} & \begin{tabular}{@{}c@{}}72.08 $\pm$ 1.93\\6\end{tabular} & \begin{tabular}{@{}c@{}}89.22 $\pm$ 5.76\\4\end{tabular} \\
\hline
ecoli & \begin{tabular}{@{}c@{}}85.32 $\pm$ 3.00\\6\end{tabular} & \begin{tabular}{@{}c@{}}85.80 $\pm$ 2.90\\3\end{tabular} & \begin{tabular}{@{}c@{}}86.05 $\pm$ 2.81\\2\end{tabular} & \begin{tabular}{@{}c@{}}85.60 $\pm$ 2.79\\5\end{tabular} & \begin{tabular}{@{}c@{}}85.71 $\pm$ 2.95\\4\end{tabular} & \begin{tabular}{@{}c@{}}\textbf{86.08 $\pm$ 2.74}\\1\end{tabular} \\
\hline
fertility & \begin{tabular}{@{}c@{}}\textbf{86.38 $\pm$ 2.01}\\1\end{tabular} & \begin{tabular}{@{}c@{}}85.24 $\pm$ 3.76\\5\end{tabular} & \begin{tabular}{@{}c@{}}84.67 $\pm$ 4.99\\6\end{tabular} & \begin{tabular}{@{}c@{}}85.33 $\pm$ 2.14\\4\end{tabular} & \begin{tabular}{@{}c@{}}86.00 $\pm$ 2.08\\2\end{tabular} & \begin{tabular}{@{}c@{}}85.62 $\pm$ 2.22\\3\end{tabular} \\
\hline
fourclass & \begin{tabular}{@{}c@{}}99.74 $\pm$ 0.33\\2\end{tabular} & \begin{tabular}{@{}c@{}}76.36 $\pm$ 2.67\\6\end{tabular} & \begin{tabular}{@{}c@{}}99.44 $\pm$ 0.45\\4\end{tabular} & \begin{tabular}{@{}c@{}}\textbf{99.93 $\pm$ 0.15}\\1\end{tabular} & \begin{tabular}{@{}c@{}}87.94 $\pm$ 1.97\\5\end{tabular} & \begin{tabular}{@{}c@{}}99.62 $\pm$ 0.35\\3\end{tabular} \\
\hline
german & \begin{tabular}{@{}c@{}}70.39 $\pm$ 1.30\\5\end{tabular} & \begin{tabular}{@{}c@{}}75.50 $\pm$ 2.03\\2\end{tabular} & \begin{tabular}{@{}c@{}}\textbf{75.53 $\pm$ 1.85}\\1\end{tabular} & \begin{tabular}{@{}c@{}}73.69 $\pm$ 2.69\\4\end{tabular} & \begin{tabular}{@{}c@{}}70.00 $\pm$ 0.00\\6\end{tabular} & \begin{tabular}{@{}c@{}}75.49 $\pm$ 1.59\\3\end{tabular} \\
\hline
glass & \begin{tabular}{@{}c@{}}69.36 $\pm$ 4.07\\3\end{tabular} & \begin{tabular}{@{}c@{}}65.41 $\pm$ 4.10\\6\end{tabular} & \begin{tabular}{@{}c@{}}70.37 $\pm$ 4.52\\2\end{tabular} & \begin{tabular}{@{}c@{}}68.66 $\pm$ 4.15\\5\end{tabular} & \begin{tabular}{@{}c@{}}68.92 $\pm$ 4.46\\4\end{tabular} & \begin{tabular}{@{}c@{}}\textbf{70.95 $\pm$ 4.42}\\1\end{tabular} \\
\hline
haberman & \begin{tabular}{@{}c@{}}72.58 $\pm$ 2.59\\4\end{tabular} & \begin{tabular}{@{}c@{}}73.04 $\pm$ 1.60\\3\end{tabular} & \begin{tabular}{@{}c@{}}73.17 $\pm$ 2.28\\2\end{tabular} & \begin{tabular}{@{}c@{}}72.45 $\pm$ 2.35\\5\end{tabular} & \begin{tabular}{@{}c@{}}\textbf{73.32 $\pm$ 2.52}\\1\end{tabular} & \begin{tabular}{@{}c@{}}71.86 $\pm$ 3.14\\6\end{tabular} \\
\hline
heart & \begin{tabular}{@{}c@{}}76.75 $\pm$ 3.84\\6\end{tabular} & \begin{tabular}{@{}c@{}}83.49 $\pm$ 3.18\\2\end{tabular} & \begin{tabular}{@{}c@{}}83.28 $\pm$ 3.69\\3\end{tabular} & \begin{tabular}{@{}c@{}}82.33 $\pm$ 3.37\\4\end{tabular} & \begin{tabular}{@{}c@{}}80.81 $\pm$ 3.68\\5\end{tabular} & \begin{tabular}{@{}c@{}}\textbf{84.09 $\pm$ 3.51}\\1\end{tabular} \\
\hline
hepatitis & \begin{tabular}{@{}c@{}}83.33 $\pm$ 0.00\\3\end{tabular} & \begin{tabular}{@{}c@{}}83.10 $\pm$ 5.54\\6\end{tabular} & \begin{tabular}{@{}c@{}}84.64 $\pm$ 4.96\\2\end{tabular} & \begin{tabular}{@{}c@{}}83.21 $\pm$ 4.40\\5\end{tabular} & \begin{tabular}{@{}c@{}}83.33 $\pm$ 0.00\\4\end{tabular} & \begin{tabular}{@{}c@{}}\textbf{85.12 $\pm$ 5.48}\\1\end{tabular} \\
\hline
ionosphere & \begin{tabular}{@{}c@{}}\textbf{94.37 $\pm$ 2.01}\\1\end{tabular} & \begin{tabular}{@{}c@{}}87.25 $\pm$ 2.73\\5\end{tabular} & \begin{tabular}{@{}c@{}}89.76 $\pm$ 2.85\\4\end{tabular} & \begin{tabular}{@{}c@{}}93.96 $\pm$ 1.97\\2\end{tabular} & \begin{tabular}{@{}c@{}}68.17 $\pm$ 2.82\\6\end{tabular} & \begin{tabular}{@{}c@{}}92.35 $\pm$ 2.09\\3\end{tabular} \\
\hline
mines & \begin{tabular}{@{}c@{}}\textbf{60.76 $\pm$ 4.01}\\1\end{tabular} & \begin{tabular}{@{}c@{}}51.51 $\pm$ 2.48\\6\end{tabular} & \begin{tabular}{@{}c@{}}60.20 $\pm$ 4.30\\2\end{tabular} & \begin{tabular}{@{}c@{}}59.92 $\pm$ 4.02\\3\end{tabular} & \begin{tabular}{@{}c@{}}53.28 $\pm$ 3.16\\5\end{tabular} & \begin{tabular}{@{}c@{}}59.47 $\pm$ 4.28\\4\end{tabular} \\
\hline
parkinson & \begin{tabular}{@{}c@{}}\textbf{92.83 $\pm$ 2.53}\\1\end{tabular} & \begin{tabular}{@{}c@{}}84.99 $\pm$ 3.95\\5\end{tabular} & \begin{tabular}{@{}c@{}}88.43 $\pm$ 3.77\\4\end{tabular} & \begin{tabular}{@{}c@{}}92.69 $\pm$ 3.31\\2\end{tabular} & \begin{tabular}{@{}c@{}}77.63 $\pm$ 2.68\\6\end{tabular} & \begin{tabular}{@{}c@{}}92.20 $\pm$ 2.75\\3\end{tabular} \\
\hline
statlog & \begin{tabular}{@{}c@{}}80.89 $\pm$ 2.38\\4\end{tabular} & \begin{tabular}{@{}c@{}}79.76 $\pm$ 2.57\\5\end{tabular} & \begin{tabular}{@{}c@{}}\textbf{83.73 $\pm$ 2.46}\\1\end{tabular} & \begin{tabular}{@{}c@{}}83.22 $\pm$ 2.30\\3\end{tabular} & \begin{tabular}{@{}c@{}}79.17 $\pm$ 2.42\\6\end{tabular} & \begin{tabular}{@{}c@{}}83.31 $\pm$ 2.07\\2\end{tabular} \\
\hline
student & \begin{tabular}{@{}c@{}}\textbf{25.71 $\pm$ 1.61}\\1\end{tabular} & \begin{tabular}{@{}c@{}}24.55 $\pm$ 4.82\\4\end{tabular} & \begin{tabular}{@{}c@{}}24.48 $\pm$ 4.89\\5\end{tabular} & \begin{tabular}{@{}c@{}}23.70 $\pm$ 3.27\\6\end{tabular} & \begin{tabular}{@{}c@{}}25.00 $\pm$ 0.00\\3\end{tabular} & \begin{tabular}{@{}c@{}}25.26 $\pm$ 5.15\\2\end{tabular} \\
\hline
vertebral & \begin{tabular}{@{}c@{}}82.33 $\pm$ 4.05\\6\end{tabular} & \begin{tabular}{@{}c@{}}\textbf{85.87 $\pm$ 3.48}\\1\end{tabular} & \begin{tabular}{@{}c@{}}85.28 $\pm$ 3.20\\2\end{tabular} & \begin{tabular}{@{}c@{}}83.75 $\pm$ 3.42\\4\end{tabular} & \begin{tabular}{@{}c@{}}83.72 $\pm$ 3.99\\5\end{tabular} & \begin{tabular}{@{}c@{}}84.61 $\pm$ 3.17\\3\end{tabular} \\
\hline
wine & \begin{tabular}{@{}c@{}}\textbf{98.68 $\pm$ 1.75}\\1\end{tabular} & \begin{tabular}{@{}c@{}}96.83 $\pm$ 2.41\\6\end{tabular} & \begin{tabular}{@{}c@{}}97.04 $\pm$ 2.47\\5\end{tabular} & \begin{tabular}{@{}c@{}}97.94 $\pm$ 1.76\\3\end{tabular} & \begin{tabular}{@{}c@{}}97.20 $\pm$ 2.19\\4\end{tabular} & \begin{tabular}{@{}c@{}}98.68 $\pm$ 1.57\\2\end{tabular} \\
\hline
\textbf{Avg. Rank} & 
3.45 & 4.30 & 3.00 & 3.60 & 4.25 & \textbf{2.40} \\
\bottomrule
\end{tabular}%
}
\end{table}

\begin{table}[htbp]
\centering
\caption{Statistics on kernel performance according to F1-score}
\label{tab:f1_results}
\resizebox{\textwidth}{!}{%
\begin{tabular}{lcccccc}
\toprule
\textbf{Kernel} & \textbf{Q-HERMITE} & \textbf{LINEAR} & \textbf{POLY} & \textbf{RBF} & \textbf{HERMITE} & \textbf{GEGEN} \\
\midrule
\textbf{Dataset} & Mean $\mid$ Std & Mean $\mid$ Std & Mean $\mid$ Std & Mean $\mid$ Std & Mean $\mid$ Std & Mean $\mid$ Std \\
\midrule
SPECTF & \begin{tabular}{@{}c@{}}73.62 $\pm$ 3.66\\4\end{tabular} & \begin{tabular}{@{}c@{}}74.84 $\pm$ 4.22\\2\end{tabular} & \begin{tabular}{@{}c@{}}74.69 $\pm$ 3.99\\3\end{tabular} & \begin{tabular}{@{}c@{}}73.61 $\pm$ 3.92\\5\end{tabular} & \begin{tabular}{@{}c@{}}69.75 $\pm$ 0.00\\6\end{tabular} & \begin{tabular}{@{}c@{}}\textbf{76.54 $\pm$ 3.63}\\1\end{tabular} \\
\hline
australian & \begin{tabular}{@{}c@{}}84.41 $\pm$ 1.85\\6\end{tabular} & \begin{tabular}{@{}c@{}}85.56 $\pm$ 1.94\\4\end{tabular} & \begin{tabular}{@{}c@{}}85.70 $\pm$ 1.76\\3\end{tabular} & \begin{tabular}{@{}c@{}}85.75 $\pm$ 2.08\\2\end{tabular} & \begin{tabular}{@{}c@{}}85.31 $\pm$ 1.41\\5\end{tabular} & \begin{tabular}{@{}c@{}}\textbf{85.86 $\pm$ 1.86}\\1\end{tabular} \\
\hline
blood & \begin{tabular}{@{}c@{}}74.67 $\pm$ 2.64\\3\end{tabular} & \begin{tabular}{@{}c@{}}66.10 $\pm$ 0.88\\6\end{tabular} & \begin{tabular}{@{}c@{}}73.67 $\pm$ 2.91\\4\end{tabular} & \begin{tabular}{@{}c@{}}\textbf{75.67 $\pm$ 1.69}\\1\end{tabular} & \begin{tabular}{@{}c@{}}69.29 $\pm$ 1.94\\5\end{tabular} & \begin{tabular}{@{}c@{}}75.12 $\pm$ 1.92\\2\end{tabular} \\
\hline
breast & \begin{tabular}{@{}c@{}}96.63 $\pm$ 1.09\\5\end{tabular} & \begin{tabular}{@{}c@{}}96.81 $\pm$ 0.76\\4\end{tabular} & \begin{tabular}{@{}c@{}}96.88 $\pm$ 1.13\\3\end{tabular} & \begin{tabular}{@{}c@{}}96.47 $\pm$ 0.90\\6\end{tabular} & \begin{tabular}{@{}c@{}}97.01 $\pm$ 0.96\\2\end{tabular} & \begin{tabular}{@{}c@{}}\textbf{97.12 $\pm$ 0.88}\\1\end{tabular} \\
\hline
cervical & \begin{tabular}{@{}c@{}}75.93 $\pm$ 9.36\\5\end{tabular} & \begin{tabular}{@{}c@{}}\textbf{91.80 $\pm$ 5.87}\\1\end{tabular} & \begin{tabular}{@{}c@{}}90.70 $\pm$ 5.35\\2\end{tabular} & \begin{tabular}{@{}c@{}}90.02 $\pm$ 4.92\\3\end{tabular} & \begin{tabular}{@{}c@{}}61.12 $\pm$ 1.85\\6\end{tabular} & \begin{tabular}{@{}c@{}}89.12 $\pm$ 5.72\\4\end{tabular} \\
\hline
ecoli & \begin{tabular}{@{}c@{}}84.54 $\pm$ 2.98\\6\end{tabular} & \begin{tabular}{@{}c@{}}85.07 $\pm$ 2.85\\3\end{tabular} & \begin{tabular}{@{}c@{}}\textbf{85.27 $\pm$ 2.80}\\1\end{tabular} & \begin{tabular}{@{}c@{}}84.74 $\pm$ 2.83\\5\end{tabular} & \begin{tabular}{@{}c@{}}84.98 $\pm$ 2.90\\4\end{tabular} & \begin{tabular}{@{}c@{}}85.16 $\pm$ 2.78\\2\end{tabular} \\
\hline
fertility & \begin{tabular}{@{}c@{}}\textbf{81.09 $\pm$ 2.26}\\1\end{tabular} & \begin{tabular}{@{}c@{}}79.80 $\pm$ 1.97\\5\end{tabular} & \begin{tabular}{@{}c@{}}79.72 $\pm$ 2.97\\6\end{tabular} & \begin{tabular}{@{}c@{}}80.31 $\pm$ 1.46\\3\end{tabular} & \begin{tabular}{@{}c@{}}80.82 $\pm$ 1.86\\2\end{tabular} & \begin{tabular}{@{}c@{}}80.16 $\pm$ 1.43\\4\end{tabular} \\
\hline
fourclass & \begin{tabular}{@{}c@{}}99.74 $\pm$ 0.33\\2\end{tabular} & \begin{tabular}{@{}c@{}}74.87 $\pm$ 4.57\\6\end{tabular} & \begin{tabular}{@{}c@{}}99.44 $\pm$ 0.44\\4\end{tabular} & \begin{tabular}{@{}c@{}}\textbf{99.93 $\pm$ 0.15}\\1\end{tabular} & \begin{tabular}{@{}c@{}}87.86 $\pm$ 2.02\\5\end{tabular} & \begin{tabular}{@{}c@{}}99.63 $\pm$ 0.35\\3\end{tabular} \\
\hline
german & \begin{tabular}{@{}c@{}}63.75 $\pm$ 2.22\\5\end{tabular} & \begin{tabular}{@{}c@{}}\textbf{74.56 $\pm$ 1.89}\\1\end{tabular} & \begin{tabular}{@{}c@{}}74.40 $\pm$ 1.68\\2\end{tabular} & \begin{tabular}{@{}c@{}}71.76 $\pm$ 4.30\\4\end{tabular} & \begin{tabular}{@{}c@{}}57.65 $\pm$ 0.00\\6\end{tabular} & \begin{tabular}{@{}c@{}}74.14 $\pm$ 1.53\\3\end{tabular} \\
\hline
glass & \begin{tabular}{@{}c@{}}68.66 $\pm$ 4.30\\3\end{tabular} & \begin{tabular}{@{}c@{}}63.34 $\pm$ 4.52\\6\end{tabular} & \begin{tabular}{@{}c@{}}69.69 $\pm$ 4.96\\2\end{tabular} & \begin{tabular}{@{}c@{}}67.81 $\pm$ 4.18\\4\end{tabular} & \begin{tabular}{@{}c@{}}66.07 $\pm$ 4.43\\5\end{tabular} & \begin{tabular}{@{}c@{}}\textbf{70.10 $\pm$ 4.33}\\1\end{tabular} \\
\hline
haberman & \begin{tabular}{@{}c@{}}66.23 $\pm$ 2.95\\3\end{tabular} & \begin{tabular}{@{}c@{}}63.79 $\pm$ 2.07\\6\end{tabular} & \begin{tabular}{@{}c@{}}\textbf{67.37 $\pm$ 3.54}\\1\end{tabular} & \begin{tabular}{@{}c@{}}65.99 $\pm$ 2.95\\5\end{tabular} & \begin{tabular}{@{}c@{}}66.42 $\pm$ 3.53\\2\end{tabular} & \begin{tabular}{@{}c@{}}66.17 $\pm$ 3.57\\4\end{tabular} \\
\hline
heart & \begin{tabular}{@{}c@{}}76.46 $\pm$ 3.95\\6\end{tabular} & \begin{tabular}{@{}c@{}}83.38 $\pm$ 3.20\\2\end{tabular} & \begin{tabular}{@{}c@{}}83.16 $\pm$ 3.72\\3\end{tabular} & \begin{tabular}{@{}c@{}}82.18 $\pm$ 3.41\\4\end{tabular} & \begin{tabular}{@{}c@{}}80.17 $\pm$ 4.00\\5\end{tabular} & \begin{tabular}{@{}c@{}}\textbf{83.96 $\pm$ 3.54}\\1\end{tabular} \\
\hline
hepatitis & \begin{tabular}{@{}c@{}}75.76 $\pm$ 0.00\\5\end{tabular} & \begin{tabular}{@{}c@{}}81.48 $\pm$ 5.81\\3\end{tabular} & \begin{tabular}{@{}c@{}}\textbf{83.26 $\pm$ 5.30}\\1\end{tabular} & \begin{tabular}{@{}c@{}}79.01 $\pm$ 5.07\\4\end{tabular} & \begin{tabular}{@{}c@{}}75.76 $\pm$ 0.00\\6\end{tabular} & \begin{tabular}{@{}c@{}}83.12 $\pm$ 6.09\\2\end{tabular} \\
\hline
ionosphere & \begin{tabular}{@{}c@{}}\textbf{94.37 $\pm$ 2.00}\\1\end{tabular} & \begin{tabular}{@{}c@{}}86.77 $\pm$ 2.97\\5\end{tabular} & \begin{tabular}{@{}c@{}}89.39 $\pm$ 3.03\\4\end{tabular} & \begin{tabular}{@{}c@{}}93.95 $\pm$ 1.97\\2\end{tabular} & \begin{tabular}{@{}c@{}}59.96 $\pm$ 4.58\\6\end{tabular} & \begin{tabular}{@{}c@{}}92.27 $\pm$ 2.09\\3\end{tabular} \\
\hline
mines & \begin{tabular}{@{}c@{}}\textbf{60.00 $\pm$ 4.33}\\1\end{tabular} & \begin{tabular}{@{}c@{}}47.28 $\pm$ 2.53\\6\end{tabular} & \begin{tabular}{@{}c@{}}59.40 $\pm$ 4.59\\2\end{tabular} & \begin{tabular}{@{}c@{}}59.25 $\pm$ 4.19\\3\end{tabular} & \begin{tabular}{@{}c@{}}49.80 $\pm$ 3.04\\5\end{tabular} & \begin{tabular}{@{}c@{}}58.42 $\pm$ 4.58\\4\end{tabular} \\
\hline
parkinson & \begin{tabular}{@{}c@{}}\textbf{92.77 $\pm$ 2.53}\\1\end{tabular} & \begin{tabular}{@{}c@{}}84.33 $\pm$ 3.94\\5\end{tabular} & \begin{tabular}{@{}c@{}}88.19 $\pm$ 3.90\\4\end{tabular} & \begin{tabular}{@{}c@{}}92.49 $\pm$ 3.57\\2\end{tabular} & \begin{tabular}{@{}c@{}}70.08 $\pm$ 5.22\\6\end{tabular} & \begin{tabular}{@{}c@{}}92.19 $\pm$ 2.71\\3\end{tabular} \\
\hline
statlog & \begin{tabular}{@{}c@{}}80.79 $\pm$ 2.44\\4\end{tabular} & \begin{tabular}{@{}c@{}}79.46 $\pm$ 2.73\\5\end{tabular} & \begin{tabular}{@{}c@{}}\textbf{83.58 $\pm$ 2.52}\\1\end{tabular} & \begin{tabular}{@{}c@{}}83.06 $\pm$ 2.36\\3\end{tabular} & \begin{tabular}{@{}c@{}}78.56 $\pm$ 2.59\\6\end{tabular} & \begin{tabular}{@{}c@{}}83.18 $\pm$ 2.11\\2\end{tabular} \\
\hline
student & \begin{tabular}{@{}c@{}}12.09 $\pm$ 2.59\\5\end{tabular} & \begin{tabular}{@{}c@{}}\textbf{22.57 $\pm$ 5.47}\\1\end{tabular} & \begin{tabular}{@{}c@{}}21.51 $\pm$ 6.14\\2\end{tabular} & \begin{tabular}{@{}c@{}}15.11 $\pm$ 4.92\\4\end{tabular} & \begin{tabular}{@{}c@{}}10.00 $\pm$ 0.00\\6\end{tabular} & \begin{tabular}{@{}c@{}}20.25 $\pm$ 7.22\\3\end{tabular} \\
\hline
vertebral & \begin{tabular}{@{}c@{}}82.56 $\pm$ 3.94\\6\end{tabular} & \begin{tabular}{@{}c@{}}\textbf{85.90 $\pm$ 3.38}\\1\end{tabular} & \begin{tabular}{@{}c@{}}85.35 $\pm$ 3.08\\2\end{tabular} & \begin{tabular}{@{}c@{}}83.89 $\pm$ 3.24\\5\end{tabular} & \begin{tabular}{@{}c@{}}83.98 $\pm$ 3.82\\4\end{tabular} & \begin{tabular}{@{}c@{}}84.63 $\pm$ 3.08\\3\end{tabular} \\
\hline
wine & \begin{tabular}{@{}c@{}}\textbf{98.68 $\pm$ 1.74}\\1\end{tabular} & \begin{tabular}{@{}c@{}}96.80 $\pm$ 2.45\\6\end{tabular} & \begin{tabular}{@{}c@{}}97.01 $\pm$ 2.51\\5\end{tabular} & \begin{tabular}{@{}c@{}}97.93 $\pm$ 1.77\\3\end{tabular} & \begin{tabular}{@{}c@{}}97.20 $\pm$ 2.17\\4\end{tabular} & \begin{tabular}{@{}c@{}}98.67 $\pm$ 1.59\\2\end{tabular} \\
\hline
\textbf{Avg. Rank} & 
3.65 & 3.90 & 2.75 & 3.45 & 4.80 & \textbf{2.45} \\
\bottomrule
\end{tabular}%
}
\end{table}

\begin{table}[htbp]
\centering
\caption{Statistics on kernel performance according to training time (seconds)}
\label{tab:training_time_results}
\resizebox{\textwidth}{!}{%
\begin{tabular}{lcccccc}
\toprule
\textbf{Kernel} & \textbf{Q-HERMITE} & \textbf{LINEAR} & \textbf{POLY} & \textbf{RBF} & \textbf{HERMITE} & \textbf{GEGEN} \\
\midrule
\textbf{Dataset} & Mean $\mid$ Std & Mean $\mid$ Std & Mean $\mid$ Std & Mean $\mid$ Std & Mean $\mid$ Std & Mean $\mid$ Std \\
\midrule
SPECTF & \begin{tabular}{@{}c@{}}0.0043 $\pm$ 0.0010\\4\end{tabular} & \begin{tabular}{@{}c@{}}0.0016 $\pm$ 0.0012\\3\end{tabular} & \begin{tabular}{@{}c@{}}\textbf{0.0013 $\pm$ 0.0004}\\1\end{tabular} & \begin{tabular}{@{}c@{}}0.0016 $\pm$ 0.0001\\2\end{tabular} & \begin{tabular}{@{}c@{}}0.0059 $\pm$ 0.0010\\5\end{tabular} & \begin{tabular}{@{}c@{}}0.0190 $\pm$ 0.0147\\6\end{tabular} \\
\hline
australian & \begin{tabular}{@{}c@{}}0.0117 $\pm$ 0.0036\\3\end{tabular} & \begin{tabular}{@{}c@{}}0.0126 $\pm$ 0.0060\\4\end{tabular} & \begin{tabular}{@{}c@{}}0.0058 $\pm$ 0.0024\\2\end{tabular} & \begin{tabular}{@{}c@{}}\textbf{0.0051 $\pm$ 0.0008}\\1\end{tabular} & \begin{tabular}{@{}c@{}}0.0200 $\pm$ 0.0012\\5\end{tabular} & \begin{tabular}{@{}c@{}}0.0359 $\pm$ 0.0049\\6\end{tabular} \\
\hline
blood & \begin{tabular}{@{}c@{}}0.0166 $\pm$ 0.0029\\4\end{tabular} & \begin{tabular}{@{}c@{}}\textbf{0.0099 $\pm$ 0.0036}\\1\end{tabular} & \begin{tabular}{@{}c@{}}0.0222 $\pm$ 0.0045\\5\end{tabular} & \begin{tabular}{@{}c@{}}0.0112 $\pm$ 0.0026\\3\end{tabular} & \begin{tabular}{@{}c@{}}0.0100 $\pm$ 0.0014\\2\end{tabular} & \begin{tabular}{@{}c@{}}0.0281 $\pm$ 0.0094\\6\end{tabular} \\
\hline
breast & \begin{tabular}{@{}c@{}}0.0079 $\pm$ 0.0011\\4\end{tabular} & \begin{tabular}{@{}c@{}}\textbf{0.0028 $\pm$ 0.0017}\\1\end{tabular} & \begin{tabular}{@{}c@{}}0.0028 $\pm$ 0.0017\\2\end{tabular} & \begin{tabular}{@{}c@{}}0.0034 $\pm$ 0.0004\\3\end{tabular} & \begin{tabular}{@{}c@{}}0.0105 $\pm$ 0.0025\\5\end{tabular} & \begin{tabular}{@{}c@{}}0.0186 $\pm$ 0.0021\\6\end{tabular} \\
\hline
cervical & \begin{tabular}{@{}c@{}}0.0012 $\pm$ 0.0001\\5\end{tabular} & \begin{tabular}{@{}c@{}}\textbf{0.0007 $\pm$ 0.0000}\\1\end{tabular} & \begin{tabular}{@{}c@{}}0.0007 $\pm$ 0.0000\\2\end{tabular} & \begin{tabular}{@{}c@{}}0.0008 $\pm$ 0.0000\\3\end{tabular} & \begin{tabular}{@{}c@{}}0.0016 $\pm$ 0.0003\\6\end{tabular} & \begin{tabular}{@{}c@{}}0.0011 $\pm$ 0.0001\\4\end{tabular} \\
\hline
ecoli & \begin{tabular}{@{}c@{}}0.0025 $\pm$ 0.0003\\4\end{tabular} & \begin{tabular}{@{}c@{}}\textbf{0.0015 $\pm$ 0.0002}\\1\end{tabular} & \begin{tabular}{@{}c@{}}0.0018 $\pm$ 0.0005\\2\end{tabular} & \begin{tabular}{@{}c@{}}0.0021 $\pm$ 0.0002\\3\end{tabular} & \begin{tabular}{@{}c@{}}0.0030 $\pm$ 0.0002\\5\end{tabular} & \begin{tabular}{@{}c@{}}0.0033 $\pm$ 0.0022\\6\end{tabular} \\
\hline
fertility & \begin{tabular}{@{}c@{}}0.0012 $\pm$ 0.0002\\4\end{tabular} & \begin{tabular}{@{}c@{}}0.0009 $\pm$ 0.0003\\3\end{tabular} & \begin{tabular}{@{}c@{}}\textbf{0.0008 $\pm$ 0.0001}\\1\end{tabular} & \begin{tabular}{@{}c@{}}0.0009 $\pm$ 0.0000\\2\end{tabular} & \begin{tabular}{@{}c@{}}0.0013 $\pm$ 0.0002\\5\end{tabular} & \begin{tabular}{@{}c@{}}0.0014 $\pm$ 0.0003\\6\end{tabular} \\
\hline
fourclass & \begin{tabular}{@{}c@{}}0.0081 $\pm$ 0.0009\\2\end{tabular} & \begin{tabular}{@{}c@{}}0.0189 $\pm$ 0.0072\\5\end{tabular} & \begin{tabular}{@{}c@{}}0.0239 $\pm$ 0.0069\\6\end{tabular} & \begin{tabular}{@{}c@{}}\textbf{0.0068 $\pm$ 0.0003}\\1\end{tabular} & \begin{tabular}{@{}c@{}}0.0104 $\pm$ 0.0003\\3\end{tabular} & \begin{tabular}{@{}c@{}}0.0118 $\pm$ 0.0007\\4\end{tabular} \\
\hline
german & \begin{tabular}{@{}c@{}}0.0398 $\pm$ 0.0045\\4\end{tabular} & \begin{tabular}{@{}c@{}}0.0344 $\pm$ 0.0130\\3\end{tabular} & \begin{tabular}{@{}c@{}}0.0218 $\pm$ 0.0103\\2\end{tabular} & \begin{tabular}{@{}c@{}}\textbf{0.0195 $\pm$ 0.0096}\\1\end{tabular} & \begin{tabular}{@{}c@{}}0.0701 $\pm$ 0.0044\\5\end{tabular} & \begin{tabular}{@{}c@{}}0.1287 $\pm$ 0.0037\\6\end{tabular} \\
\hline
glass & \begin{tabular}{@{}c@{}}0.0024 $\pm$ 0.0007\\6\end{tabular} & \begin{tabular}{@{}c@{}}0.0018 $\pm$ 0.0006\\2\end{tabular} & \begin{tabular}{@{}c@{}}0.0021 $\pm$ 0.0003\\3\end{tabular} & \begin{tabular}{@{}c@{}}\textbf{0.0016 $\pm$ 0.0001}\\1\end{tabular} & \begin{tabular}{@{}c@{}}0.0022 $\pm$ 0.0001\\4\end{tabular} & \begin{tabular}{@{}c@{}}0.0024 $\pm$ 0.0003\\5\end{tabular} \\
\hline
haberman & \begin{tabular}{@{}c@{}}0.0028 $\pm$ 0.0015\\4\end{tabular} & \begin{tabular}{@{}c@{}}\textbf{0.0018 $\pm$ 0.0010}\\1\end{tabular} & \begin{tabular}{@{}c@{}}0.0038 $\pm$ 0.0020\\6\end{tabular} & \begin{tabular}{@{}c@{}}0.0020 $\pm$ 0.0005\\3\end{tabular} & \begin{tabular}{@{}c@{}}0.0019 $\pm$ 0.0004\\2\end{tabular} & \begin{tabular}{@{}c@{}}0.0037 $\pm$ 0.0016\\5\end{tabular} \\
\hline
heart & \begin{tabular}{@{}c@{}}0.0027 $\pm$ 0.0003\\3\end{tabular} & \begin{tabular}{@{}c@{}}0.0037 $\pm$ 0.0022\\6\end{tabular} & \begin{tabular}{@{}c@{}}\textbf{0.0014 $\pm$ 0.0005}\\1\end{tabular} & \begin{tabular}{@{}c@{}}0.0015 $\pm$ 0.0001\\2\end{tabular} & \begin{tabular}{@{}c@{}}0.0032 $\pm$ 0.0001\\5\end{tabular} & \begin{tabular}{@{}c@{}}0.0028 $\pm$ 0.0003\\4\end{tabular} \\
\hline
hepatitis & \begin{tabular}{@{}c@{}}0.0018 $\pm$ 0.0004\\6\end{tabular} & \begin{tabular}{@{}c@{}}\textbf{0.0008 $\pm$ 0.0001}\\1\end{tabular} & \begin{tabular}{@{}c@{}}0.0008 $\pm$ 0.0000\\2\end{tabular} & \begin{tabular}{@{}c@{}}0.0009 $\pm$ 0.0001\\3\end{tabular} & \begin{tabular}{@{}c@{}}0.0016 $\pm$ 0.0003\\5\end{tabular} & \begin{tabular}{@{}c@{}}0.0013 $\pm$ 0.0001\\4\end{tabular} \\
\hline
ionosphere & \begin{tabular}{@{}c@{}}0.0052 $\pm$ 0.0003\\4\end{tabular} & \begin{tabular}{@{}c@{}}0.0027 $\pm$ 0.0016\\3\end{tabular} & \begin{tabular}{@{}c@{}}0.0021 $\pm$ 0.0005\\2\end{tabular} & \begin{tabular}{@{}c@{}}\textbf{0.0020 $\pm$ 0.0001}\\1\end{tabular} & \begin{tabular}{@{}c@{}}0.0078 $\pm$ 0.0006\\5\end{tabular} & \begin{tabular}{@{}c@{}}0.0108 $\pm$ 0.0008\\6\end{tabular} \\
\hline
mines & \begin{tabular}{@{}c@{}}0.0071 $\pm$ 0.0019\\5\end{tabular} & \begin{tabular}{@{}c@{}}0.0035 $\pm$ 0.0018\\2\end{tabular} & \begin{tabular}{@{}c@{}}0.0120 $\pm$ 0.0026\\6\end{tabular} & \begin{tabular}{@{}c@{}}0.0049 $\pm$ 0.0006\\3\end{tabular} & \begin{tabular}{@{}c@{}}\textbf{0.0031 $\pm$ 0.0004}\\1\end{tabular} & \begin{tabular}{@{}c@{}}0.0055 $\pm$ 0.0014\\4\end{tabular} \\
\hline
parkinson & \begin{tabular}{@{}c@{}}0.0026 $\pm$ 0.0004\\4\end{tabular} & \begin{tabular}{@{}c@{}}0.0021 $\pm$ 0.0008\\3\end{tabular} & \begin{tabular}{@{}c@{}}0.0016 $\pm$ 0.0003\\2\end{tabular} & \begin{tabular}{@{}c@{}}\textbf{0.0012 $\pm$ 0.0000}\\1\end{tabular} & \begin{tabular}{@{}c@{}}0.0032 $\pm$ 0.0002\\6\end{tabular} & \begin{tabular}{@{}c@{}}0.0030 $\pm$ 0.0003\\5\end{tabular} \\
\hline
statlog & \begin{tabular}{@{}c@{}}0.0276 $\pm$ 0.0032\\4\end{tabular} & \begin{tabular}{@{}c@{}}0.0180 $\pm$ 0.0054\\2\end{tabular} & \begin{tabular}{@{}c@{}}0.0200 $\pm$ 0.0035\\3\end{tabular} & \begin{tabular}{@{}c@{}}\textbf{0.0124 $\pm$ 0.0009}\\1\end{tabular} & \begin{tabular}{@{}c@{}}0.0383 $\pm$ 0.0009\\5\end{tabular} & \begin{tabular}{@{}c@{}}0.0751 $\pm$ 0.0038\\6\end{tabular} \\
\hline
student & \begin{tabular}{@{}c@{}}0.0025 $\pm$ 0.0005\\4\end{tabular} & \begin{tabular}{@{}c@{}}0.0018 $\pm$ 0.0004\\3\end{tabular} & \begin{tabular}{@{}c@{}}\textbf{0.0015 $\pm$ 0.0002}\\1\end{tabular} & \begin{tabular}{@{}c@{}}0.0016 $\pm$ 0.0001\\2\end{tabular} & \begin{tabular}{@{}c@{}}0.0029 $\pm$ 0.0005\\5\end{tabular} & \begin{tabular}{@{}c@{}}0.0033 $\pm$ 0.0012\\6\end{tabular} \\
\hline
vertebral & \begin{tabular}{@{}c@{}}0.0024 $\pm$ 0.0007\\4\end{tabular} & \begin{tabular}{@{}c@{}}0.0017 $\pm$ 0.0004\\2\end{tabular} & \begin{tabular}{@{}c@{}}\textbf{0.0017 $\pm$ 0.0004}\\1\end{tabular} & \begin{tabular}{@{}c@{}}0.0018 $\pm$ 0.0001\\3\end{tabular} & \begin{tabular}{@{}c@{}}0.0025 $\pm$ 0.0002\\5\end{tabular} & \begin{tabular}{@{}c@{}}0.0030 $\pm$ 0.0019\\6\end{tabular} \\
\hline
wine & \begin{tabular}{@{}c@{}}0.0019 $\pm$ 0.0003\\5\end{tabular} & \begin{tabular}{@{}c@{}}\textbf{0.0009 $\pm$ 0.0000}\\1\end{tabular} & \begin{tabular}{@{}c@{}}0.0010 $\pm$ 0.0001\\2\end{tabular} & \begin{tabular}{@{}c@{}}0.0011 $\pm$ 0.0001\\3\end{tabular} & \begin{tabular}{@{}c@{}}0.0019 $\pm$ 0.0002\\4\end{tabular} & \begin{tabular}{@{}c@{}}0.0020 $\pm$ 0.0002\\6\end{tabular} \\
\hline
\textbf{Avg. Rank} & 
4.15 & 2.40 & 2.60 & \textbf{2.10} & 4.40 & 5.35 \\
\bottomrule
\end{tabular}%
}
\end{table}

\begin{table}[htbp]
\centering
\caption{Statistics on kernel performance according to support vector proportion}
\label{tab:support_vector_results}
\resizebox{\textwidth}{!}{%
\begin{tabular}{lcccccc}
\toprule
\textbf{Kernel} & \textbf{Q-HERMITE} & \textbf{LINEAR} & \textbf{POLY} & \textbf{RBF} & \textbf{HERMITE} & \textbf{GEGEN} \\
\midrule
\textbf{Dataset} & Mean $\mid$ Std & Mean $\mid$ Std & Mean $\mid$ Std & Mean $\mid$ Std & Mean $\mid$ Std & Mean $\mid$ Std \\
\midrule
SPECTF & \begin{tabular}{@{}c@{}}65.5607 $\pm$ 18.0682\\4\end{tabular} & \begin{tabular}{@{}c@{}}\textbf{40.4455 $\pm$ 5.3811}\\1\end{tabular} & \begin{tabular}{@{}c@{}}43.3333 $\pm$ 2.8134\\2\end{tabular} & \begin{tabular}{@{}c@{}}80.0000 $\pm$ 22.2516\\6\end{tabular} & \begin{tabular}{@{}c@{}}50.4762 $\pm$ 6.7283\\3\end{tabular} & \begin{tabular}{@{}c@{}}71.7665 $\pm$ 18.9767\\5\end{tabular} \\
\hline
australian & \begin{tabular}{@{}c@{}}55.9243 $\pm$ 5.0646\\6\end{tabular} & \begin{tabular}{@{}c@{}}38.0657 $\pm$ 4.5306\\4\end{tabular} & \begin{tabular}{@{}c@{}}35.4451 $\pm$ 11.7165\\2\end{tabular} & \begin{tabular}{@{}c@{}}\textbf{35.1849 $\pm$ 4.8312}\\1\end{tabular} & \begin{tabular}{@{}c@{}}50.9021 $\pm$ 2.3204\\5\end{tabular} & \begin{tabular}{@{}c@{}}36.9417 $\pm$ 12.0825\\3\end{tabular} \\
\hline
blood & \begin{tabular}{@{}c@{}}51.4395 $\pm$ 2.8761\\5\end{tabular} & \begin{tabular}{@{}c@{}}48.9047 $\pm$ 0.5710\\2\end{tabular} & \begin{tabular}{@{}c@{}}53.4390 $\pm$ 3.5207\\6\end{tabular} & \begin{tabular}{@{}c@{}}49.2762 $\pm$ 1.5011\\3\end{tabular} & \begin{tabular}{@{}c@{}}\textbf{48.5824 $\pm$ 1.0111}\\1\end{tabular} & \begin{tabular}{@{}c@{}}50.1830 $\pm$ 1.9345\\4\end{tabular} \\
\hline
breast & \begin{tabular}{@{}c@{}}29.7609 $\pm$ 5.6359\\6\end{tabular} & \begin{tabular}{@{}c@{}}\textbf{9.1632 $\pm$ 2.2686}\\1\end{tabular} & \begin{tabular}{@{}c@{}}10.6934 $\pm$ 2.9296\\2\end{tabular} & \begin{tabular}{@{}c@{}}22.3849 $\pm$ 14.7848\\5\end{tabular} & \begin{tabular}{@{}c@{}}19.3664 $\pm$ 3.3264\\4\end{tabular} & \begin{tabular}{@{}c@{}}16.5272 $\pm$ 7.2075\\3\end{tabular} \\
\hline
cervical & \begin{tabular}{@{}c@{}}79.7714 $\pm$ 8.4010\\6\end{tabular} & \begin{tabular}{@{}c@{}}\textbf{30.4000 $\pm$ 6.7579}\\1\end{tabular} & \begin{tabular}{@{}c@{}}38.8571 $\pm$ 10.1112\\2\end{tabular} & \begin{tabular}{@{}c@{}}51.1429 $\pm$ 14.5664\\3\end{tabular} & \begin{tabular}{@{}c@{}}78.4571 $\pm$ 5.4319\\5\end{tabular} & \begin{tabular}{@{}c@{}}53.7143 $\pm$ 16.1007\\4\end{tabular} \\
\hline
ecoli & \begin{tabular}{@{}c@{}}47.5988 $\pm$ 4.7908\\4\end{tabular} & \begin{tabular}{@{}c@{}}44.4134 $\pm$ 6.9819\\2\end{tabular} & \begin{tabular}{@{}c@{}}\textbf{42.4195 $\pm$ 5.2164}\\1\end{tabular} & \begin{tabular}{@{}c@{}}47.8541 $\pm$ 6.2411\\5\end{tabular} & \begin{tabular}{@{}c@{}}50.0790 $\pm$ 4.4511\\6\end{tabular} & \begin{tabular}{@{}c@{}}47.5745 $\pm$ 5.3680\\3\end{tabular} \\
\hline
fertility & \begin{tabular}{@{}c@{}}46.4490 $\pm$ 27.1189\\4\end{tabular} & \begin{tabular}{@{}c@{}}42.1633 $\pm$ 16.3146\\2\end{tabular} & \begin{tabular}{@{}c@{}}\textbf{39.8367 $\pm$ 9.9314}\\1\end{tabular} & \begin{tabular}{@{}c@{}}94.3673 $\pm$ 11.3925\\6\end{tabular} & \begin{tabular}{@{}c@{}}42.8980 $\pm$ 5.4370\\3\end{tabular} & \begin{tabular}{@{}c@{}}65.2245 $\pm$ 22.0494\\5\end{tabular} \\
\hline
fourclass & \begin{tabular}{@{}c@{}}8.8652 $\pm$ 1.7924\\3\end{tabular} & \begin{tabular}{@{}c@{}}52.2909 $\pm$ 1.3720\\6\end{tabular} & \begin{tabular}{@{}c@{}}14.2241 $\pm$ 0.7870\\4\end{tabular} & \begin{tabular}{@{}c@{}}6.5198 $\pm$ 2.1181\\2\end{tabular} & \begin{tabular}{@{}c@{}}40.2180 $\pm$ 1.2349\\5\end{tabular} & \begin{tabular}{@{}c@{}}\textbf{5.6053 $\pm$ 1.7374}\\1\end{tabular} \\
\hline
german & \begin{tabular}{@{}c@{}}95.9265 $\pm$ 2.8937\\6\end{tabular} & \begin{tabular}{@{}c@{}}59.9143 $\pm$ 5.1912\\3\end{tabular} & \begin{tabular}{@{}c@{}}\textbf{53.8122 $\pm$ 2.1492}\\1\end{tabular} & \begin{tabular}{@{}c@{}}61.2980 $\pm$ 14.7277\\5\end{tabular} & \begin{tabular}{@{}c@{}}60.3714 $\pm$ 1.1056\\4\end{tabular} & \begin{tabular}{@{}c@{}}55.6245 $\pm$ 2.0993\\2\end{tabular} \\
\hline
glass & \begin{tabular}{@{}c@{}}69.9521 $\pm$ 5.8334\\3\end{tabular} & \begin{tabular}{@{}c@{}}75.3595 $\pm$ 4.7985\\5\end{tabular} & \begin{tabular}{@{}c@{}}\textbf{68.2454 $\pm$ 4.3132}\\1\end{tabular} & \begin{tabular}{@{}c@{}}75.3020 $\pm$ 3.8896\\4\end{tabular} & \begin{tabular}{@{}c@{}}80.5753 $\pm$ 3.8234\\6\end{tabular} & \begin{tabular}{@{}c@{}}69.6836 $\pm$ 6.4043\\2\end{tabular} \\
\hline
haberman & \begin{tabular}{@{}c@{}}\textbf{52.5234 $\pm$ 2.5986}\\1\end{tabular} & \begin{tabular}{@{}c@{}}55.2069 $\pm$ 2.5254\\6\end{tabular} & \begin{tabular}{@{}c@{}}53.6849 $\pm$ 2.7058\\3\end{tabular} & \begin{tabular}{@{}c@{}}55.1535 $\pm$ 3.4977\\5\end{tabular} & \begin{tabular}{@{}c@{}}53.5648 $\pm$ 2.4116\\2\end{tabular} & \begin{tabular}{@{}c@{}}54.3124 $\pm$ 3.0120\\4\end{tabular} \\
\hline
heart & \begin{tabular}{@{}c@{}}81.1338 $\pm$ 4.5285\\6\end{tabular} & \begin{tabular}{@{}c@{}}\textbf{40.7407 $\pm$ 6.6819}\\1\end{tabular} & \begin{tabular}{@{}c@{}}44.1119 $\pm$ 10.3860\\2\end{tabular} & \begin{tabular}{@{}c@{}}46.5306 $\pm$ 9.9243\\4\end{tabular} & \begin{tabular}{@{}c@{}}64.9887 $\pm$ 2.0840\\5\end{tabular} & \begin{tabular}{@{}c@{}}44.5805 $\pm$ 7.8523\\3\end{tabular} \\
\hline
hepatitis & \begin{tabular}{@{}c@{}}65.3571 $\pm$ 32.3440\\5\end{tabular} & \begin{tabular}{@{}c@{}}\textbf{32.5510 $\pm$ 4.7742}\\1\end{tabular} & \begin{tabular}{@{}c@{}}45.7653 $\pm$ 9.4596\\3\end{tabular} & \begin{tabular}{@{}c@{}}78.8265 $\pm$ 21.1812\\6\end{tabular} & \begin{tabular}{@{}c@{}}38.8265 $\pm$ 10.7002\\2\end{tabular} & \begin{tabular}{@{}c@{}}58.7755 $\pm$ 19.3148\\4\end{tabular} \\
\hline
ionosphere & \begin{tabular}{@{}c@{}}65.3644 $\pm$ 2.3256\\6\end{tabular} & \begin{tabular}{@{}c@{}}27.6851 $\pm$ 8.1530\\2\end{tabular} & \begin{tabular}{@{}c@{}}\textbf{26.4840 $\pm$ 3.8485}\\1\end{tabular} & \begin{tabular}{@{}c@{}}50.0175 $\pm$ 9.6225\\4\end{tabular} & \begin{tabular}{@{}c@{}}65.0496 $\pm$ 2.6036\\5\end{tabular} & \begin{tabular}{@{}c@{}}44.4315 $\pm$ 12.7619\\3\end{tabular} \\
\hline
mines & \begin{tabular}{@{}c@{}}72.5061 $\pm$ 2.6155\\2\end{tabular} & \begin{tabular}{@{}c@{}}79.1768 $\pm$ 2.1081\\4\end{tabular} & \begin{tabular}{@{}c@{}}\textbf{71.6586 $\pm$ 2.6155}\\1\end{tabular} & \begin{tabular}{@{}c@{}}79.4673 $\pm$ 3.4637\\5\end{tabular} & \begin{tabular}{@{}c@{}}80.5932 $\pm$ 1.6668\\6\end{tabular} & \begin{tabular}{@{}c@{}}76.3196 $\pm$ 3.8058\\3\end{tabular} \\
\hline
parkinson & \begin{tabular}{@{}c@{}}47.5840 $\pm$ 6.4571\\4\end{tabular} & \begin{tabular}{@{}c@{}}31.6176 $\pm$ 4.0807\\2\end{tabular} & \begin{tabular}{@{}c@{}}\textbf{28.8235 $\pm$ 2.8764}\\1\end{tabular} & \begin{tabular}{@{}c@{}}60.2311 $\pm$ 15.1858\\6\end{tabular} & \begin{tabular}{@{}c@{}}54.1597 $\pm$ 1.7660\\5\end{tabular} & \begin{tabular}{@{}c@{}}35.9244 $\pm$ 4.4575\\3\end{tabular} \\
\hline
statlog & \begin{tabular}{@{}c@{}}57.8777 $\pm$ 6.2806\\6\end{tabular} & \begin{tabular}{@{}c@{}}47.4643 $\pm$ 3.5915\\2\end{tabular} & \begin{tabular}{@{}c@{}}\textbf{44.7136 $\pm$ 2.8573}\\1\end{tabular} & \begin{tabular}{@{}c@{}}49.6737 $\pm$ 3.0474\\4\end{tabular} & \begin{tabular}{@{}c@{}}56.6981 $\pm$ 2.0120\\5\end{tabular} & \begin{tabular}{@{}c@{}}47.8608 $\pm$ 2.2305\\3\end{tabular} \\
\hline
student & \begin{tabular}{@{}c@{}}99.6888 $\pm$ 1.1780\\5\end{tabular} & \begin{tabular}{@{}c@{}}\textbf{98.6139 $\pm$ 1.1152}\\1\end{tabular} & \begin{tabular}{@{}c@{}}99.4625 $\pm$ 0.9271\\3\end{tabular} & \begin{tabular}{@{}c@{}}100.0000 $\pm$ 0.0000\\6\end{tabular} & \begin{tabular}{@{}c@{}}99.3494 $\pm$ 0.7463\\2\end{tabular} & \begin{tabular}{@{}c@{}}99.5474 $\pm$ 1.0137\\4\end{tabular} \\
\hline
vertebral & \begin{tabular}{@{}c@{}}42.5938 $\pm$ 6.3380\\3\end{tabular} & \begin{tabular}{@{}c@{}}\textbf{37.4194 $\pm$ 4.3227}\\1\end{tabular} & \begin{tabular}{@{}c@{}}40.4345 $\pm$ 4.5398\\2\end{tabular} & \begin{tabular}{@{}c@{}}43.9105 $\pm$ 3.9297\\5\end{tabular} & \begin{tabular}{@{}c@{}}44.7268 $\pm$ 4.8632\\6\end{tabular} & \begin{tabular}{@{}c@{}}43.7130 $\pm$ 4.3119\\4\end{tabular} \\
\hline
wine & \begin{tabular}{@{}c@{}}50.3226 $\pm$ 7.1385\\5\end{tabular} & \begin{tabular}{@{}c@{}}33.8249 $\pm$ 17.1291\\2\end{tabular} & \begin{tabular}{@{}c@{}}38.9401 $\pm$ 20.2744\\3\end{tabular} & \begin{tabular}{@{}c@{}}41.7512 $\pm$ 15.4921\\4\end{tabular} & \begin{tabular}{@{}c@{}}\textbf{31.9816 $\pm$ 7.5182}\\1\end{tabular} & \begin{tabular}{@{}c@{}}54.2627 $\pm$ 17.3325\\6\end{tabular} \\
\hline
\textbf{Avg. Rank} & 
4.50 & 2.45 & \textbf{2.10} & 4.45 & 4.05 & 3.45 \\
\bottomrule
\end{tabular}%
}
\end{table}

\subsection{Study limitations}


While the metrics employed in this study offer useful insights into kernel performance, they may only partially reflect the adaptability of kernels across varied scenarios and provide limited indications of the specific conditions under which each kernel excels. These evaluations effectively characterize kernel behavior on the selected datasets, though the restricted collection may introduce some bias in global interpretations. A promising direction for future work could involve incorporating additional metrics, such as those based on spectral analysis of the kernel matrix, to enable a more comprehensive understanding of kernels \cite{ li2017efficient, liu2017infinite}. As observed in our results, the proportion of support vectors, at least for the dataset collection considered here, does not support Padierna et al.'s hypotheses \cite{padierna2018novel} regarding substantial improvements over classical kernels. In particular, their claim that orthogonal polynomial kernels reduce redundancy in the feature-space representation, thereby yielding fewer support vectors, is not reflected in our empirical findings. Nevertheless, because these outcomes are strongly dataset-dependent and no formal analytical characterization of this hypothesis has been established, it remains difficult to draw definitive conclusions about the expected behaviour of PSV across broader settings, leaving room for further in-depth studies on this topic.

Rank-based comparisons may also introduce interpretive biases, as the absolute differences between metrics (e.g., accuracy or F1-score) are often small, even when Friedman tests consider them statistically significant \cite{liu2022t}. This limitation is well-documented in the statistical literature \cite{flach2019performance, demvsar2006statistical, garcia2010advanced}, where the dissociation between statistical significance and practical significance frequently occurs in machine learning comparisons. This phenomenon is particularly pronounced in high-dimensional comparison scenarios where multiple classifiers are evaluated across numerous datasets, as small but consistent performance differences can yield significant rank differences while offering negligible practical utility \cite{benavoli2017time, bouthillier2021accounting}. The issue is further compounded by the sensitivity of non-parametric tests to sample size, where large-scale experimental designs (e.g., 35 trials across multiple datasets) can detect minuscule differences that lack real-world relevance \cite{wasserstein2016ASA}.

Moreover, our analysis does not include an explicit optimization of semantic hyperparameters appearing in the kernel definitions. Both the choice of $\varepsilon = 0.1$ as a lower clipping bound for the Gegenbauer and q-Hermite weight functions, as well as the truncation of the Pochhammer operator at $\vartheta = 10$, remain design decisions associated with a data science process. These parameters clearly affect the effective resolution and numerical behaviour of the kernels, so a specific sensitivity analysis is recommended in future work.

A priority for future research is to develop new metrics that better characterize the algebraic structure of these kernels and their capacity to adapt to different feature spaces. The experimental design could also be expanded to support these metrics; for example, explicitly studying the influence of the number of classes, features, or samples on the kernel behaviour \cite{uddin2024dataset, althnian2021impact}. Such analyses would allow us to determine the application regimes in which each kernel family is most effective, reducing the dependence on a small and heterogeneous set of datasets whose idiosyncrasies can distort broader conclusions \cite{lones2024avoiding, guido2024overview}.

\section{Conclusions}
\label{secconc}

In this work, we have introduced a novel family of kernel functions for Support Vector Machines based on $q$-orthogonal polynomials, specifically the discrete $q$-Hermite I polynomials. This represents a significant extension beyond the established Gegenbauer paradigm in orthogonal polynomial kernels for machine learning.

The theoretical foundation of our approach demonstrates several advantageous properties. Unlike traditional orthogonal polynomial kernels that require explicit scaling functions to constrain polynomial amplitudes outside the natural domain $[-1,1]$, the q-Hermite kernel exhibits inherent boundedness properties. We have rigorously proven that for the parameter choice $a = -1$ and $0 < q < 1$, the discrete $q$-Hermite I polynomials remain uniformly bounded over the compact domain $[-1,1]$. This inherent boundedness, combined with the carefully constructed weight function that remains strictly within the interval $(0,1]$, effectively prevents both annihilation and explosion effects that commonly plague orthogonal polynomial kernels.

Furthermore, we have established that the q-Hermite kernel satisfies Mercer's conditions, ensuring its validity for use in SVM optimization. The kernel's symmetry and positive semidefiniteness guarantee its correspondence to an implicit feature space where SVMs can operate effectively. The multiplicative extension to higher dimensions preserves these properties, making the kernel applicable to real-world classification problems.

From a computational perspective, the elimination of the scaling function requirement presents advantages in both implementation simplicity and potential efficiency gains. The analytical properties of the discrete $q$-Hermite I polynomials provide natural regularization against numerical instability. These computational advantages are particularly salient in high-dimensional or noisy data regimes, where the q-deformation parameter provides tunable adaptability without the numerical instabilities common in non-deformed polynomials. 

A promising direction for future research is to design kernels for QSVMs using discrete $q$-Hermite I polynomials. General quantum kernels of QSVMs can offer decisive advantages over their classical counterparts by efficiently exploring high-dimensional Hilbert spaces, which provides great expressiveness to capture complex nonlinear relationships without the computational costs of classical systems. In this sense, q-orthogonal polynomials could result very adequate for the kernels of QSVMs as they emerged in the context of quantum mechanics and have unique properties relevant to the description of complex quantum systems and quantum groups. Besides, the q-deformation parameter allows for a gradual transition between classical and quantum implementations, making them well-suited in hybrid quantum-classical settings.

Overall, this research not only enriches the orthogonal polynomial kernel landscape but also underscores the value of q-analogues in providing mathematically elegant, stable, and efficient solutions for pattern recognition tasks. The open-source implementation further facilitates adoption and extension by the community.





\section*{Funding sources}
\label{fund}
This research was supported by the Universidad de Alcalá through Ayudas para la realización de proyectos de investigación UAH, 2024.  The authors gratefully acknowledge the financial support provided under project QuSuVeMa, PIUAH24/IA-021.

\section*{Data Availability}
\label{data}
The authors provide full code and experimental resources at
\url{https://github.com/Kokechacho/SVMs-QSVMs}.

\bibliographystyle{elsarticle-num}
\bibliography{SVMs}

\appendix

\begin{table}[htbp]
\centering
\caption{Average ranks from Friedman tests (the lower values the better). Kernels are sorted by average rank for each metric.}
\label{tab:friedman_ranks}
\small
\begin{tabular}{lcccccc}
\hline
\textbf{Metric} & \multicolumn{6}{c}{\textbf{Average Rank}} \\
& \textbf{Linear} & \textbf{Hermite} & \textbf{RBF} & \textbf{Q-Hermite} & \textbf{Polynomial} & \textbf{Gegenbauer} \\
\hline
Accuracy & 2.700 & 2.775 & 3.400 & 3.525 & 4.000 & 4.600 \\
Training Time & 2.400 & 4.400 & \textbf{2.100} & 4.150 & 2.600 & 5.350 \\
F1-Score & 3.100 & \textbf{2.225} & 3.550 & 3.325 & 4.250 & 4.550 \\
SPV \% & 2.450 & 4.050 & 4.450 & 4.500 & \textbf{2.100} & 3.450 \\
\hline
\end{tabular}
\vspace{0.1cm}
\begin{minipage}{\textwidth}
\footnotesize
\textbf{Note:} Bold values indicate the best performing kernel for each metric.
\end{minipage}
\end{table}

\begin{table}[htbp]
\centering
\caption{Pairwise comparisons of Q-Hermite kernel performance against other kernel groups. Positive mean differences favor Q-Hermite.}
\label{tab:pairwise_comparisons}
\small
\begin{tabular}{lcccc}
\hline
\textbf{Metric} & \textbf{Comparison} & \textbf{Wilcoxon W} & \textbf{p-value} & \textbf{Mean Difference} \\
\hline
\multirow{3}{*}{Accuracy} & vs All Kernels & 100.00 & 0.8695 & 0.0031 \\
 & vs Classical Kernels & 100.00 & 0.8695 & -0.0013 \\
 & vs Orthogonal Kernels & 103.50 & 0.9553 & 0.0096 \\
\hline
\multirow{3}{*}{Training Time} & vs All Kernels & 90.00 & 0.5958 & -0.0013 \\
 & vs Classical Kernels & \textbf{18.00} & \textbf{0.0005} & \textbf{0.0021} \\
 & vs Orthogonal Kernels & \textbf{27.00} & \textbf{0.0023} & \textbf{-0.0064} \\
\hline
\multirow{3}{*}{F1-Score} & vs All Kernels & 104.00 & 0.9854 & -0.0006 \\
 & vs Classical Kernels & 93.00 & 0.6742 & -0.0106 \\
 & vs Orthogonal Kernels & 89.00 & 0.5706 & 0.0145 \\
\hline
\multirow{3}{*}{SPV \%} & vs All Kernels & \textbf{45.00} & \textbf{0.0240} & \textbf{8.3787} \\
 & vs Classical Kernels & \textbf{43.00} & \textbf{0.0192} & \textbf{10.1650} \\
 & vs Orthogonal Kernels & 56.00 & 0.0696 & 5.6991 \\
\hline
\end{tabular}
\vspace{0.1cm}
\begin{minipage}{\textwidth}
\footnotesize
\textbf{Note:} Significant comparisons ($p < 0.05$) are shown in bold. Classical kernels: Linear, Polynomial, RBF. Orthogonal kernels: Hermite, Gegenbauer. Positive differences indicate Q-Hermite performs better.
\end{minipage}
\end{table}

\end{document}